\documentclass[twoside,11pt]{article}

\usepackage{jmlr2e}

\usepackage{graphicx}
\usepackage{latexsym}
\usepackage{amsmath}
\usepackage{url}

\usepackage[pdftex,bookmarks,colorlinks,breaklinks]{hyperref}  
\usepackage{color}
\definecolor{dullmagenta}{rgb}{0.4,0,0.4}   
\definecolor{darkblue}{rgb}{0,0,0.4}
\hypersetup{linkcolor=red,citecolor=blue,filecolor=dullmagenta,urlcolor=darkblue} 

\usepackage{algorithm}
\usepackage{algorithmic}

\usepackage{tikz}
\usepackage{array}
\newcolumntype{S}{>{\centering\arraybackslash} m{.16\linewidth} }

\newtheorem{thm}{Theorem}[section]

\newtheorem{prop}[thm]{Proposition}

\newcommand{\ie}{{\it i.e.}}
\newcommand{\eg}{{\it e.g.}}
\newcommand{\E}{{\mathbb E}}
\newcommand{\pr}{\mathrm{pr}}

\jmlrheading{}{2014}{}{}{}{Braxton Osting, Christoph Brune, and Stanley J. Osher}
\ShortHeadings{Optimal Data collection For Informative Rankings}{Osting, Brune, and Osher}
\firstpageno{1}

\begin{document}
\title{Optimal Data Collection For Informative \\  Rankings Expose Well-Connected Graphs}

\author{\name Braxton Osting \email osting@math.utah.edu\\%
\addr Department of Mathematics \\ 
University of Utah \\
Salt Lake City, UT 84112, USA 
\AND
\name Christoph Brune	 \email c.brune@utwente.nl \\%
\addr 
Department of Applied Mathematics\\ 
University of Twente \\ 
7500 AE Enschede, The Netherlands
\AND
\name Stanley J.  Osher \email sjo@math.ucla.edu\\%
\addr Department of Mathematics \\ University of California \\
Los Angeles, CA 90095, USA}%

\editor{}

\maketitle

\begin{abstract}%
Given a graph where vertices represent alternatives and arcs represent pairwise comparison data, the statistical ranking problem is to find a potential function, defined on the vertices, such that the gradient of the potential function agrees with the pairwise comparisons. Our goal in this paper is to develop a method for collecting data for which the least squares estimator for the ranking problem has maximal Fisher information. Our approach, based on experimental design, is to view data collection as a bi-level optimization problem where the inner problem is the ranking problem and the outer problem is to identify data which maximizes the informativeness of the ranking. Under certain assumptions, the data collection problem decouples, reducing to a problem of finding multigraphs with large algebraic connectivity. This reduction of the data collection problem to graph-theoretic questions is one of the primary contributions of this work. As an application, we study the Yahoo! Movie user rating dataset and demonstrate that the addition of a small number of well-chosen pairwise comparisons can significantly increase the Fisher informativeness of the ranking. As another application, we study the 2011-12 NCAA football schedule and propose schedules with the same number of games which are significantly more informative. Using spectral clustering methods to identify highly-connected communities within the division, we argue that the NCAA could improve its notoriously poor rankings by simply scheduling more out-of-conference games.
\end{abstract}

\bigskip


\section{Introduction} \label{sec:intro}
The problem of statistical ranking\footnote{We use the term ranking to indicate a numerical score for each item in a collection, which is also sometimes referred to as a rating. } arises in a variety of applications, where a collection of alternatives  is  ranked based on pairwise comparisons. 
Methods for ranking must address a number of inherent difficulties including incomplete,  inconsistent, and imbalanced data. 
Despite and possibly as a consequence of these difficulties, although ranking from pairwise comparison data is an old problem \citep{David:1963ys}, there have been several recent contributions to the subject  with applications in social networking, game theory, e-commerce, and logistics \citep{Langville:2012,Osting:2012fk,Hirani:2011fk,Jiang2010,Callaghan:2007}. 

The  statistical ranking problem can be generally posed as finding an estimate for a  ranking, $\phi$, for a set of alternatives from a dataset which consists of (i) a set of alternative pairs which have been queried, $w$, and (ii) noisy, cardinal\footnote{A \emph{cardinal} pairwise comparison dataset refers to quantitative  (real-valued) comparisons between items, as opposed to an \emph{ordinal} pairwise comparison dataset, where only pairwise preferences are specified. } pairwise comparisons for those alternative pairs, $y$. 
We symbolically express an estimator for the ranking problem,
\begin{equation}
\label{eq:R}
\hat \phi_{w} = \mathcal{R}(y,w),
\end{equation}
where the dependence of the  ranking, $\hat \phi_{w}$, on the queried pairs (data collected), $w$, is emphasized by the subscript. 

Consider the dependence of a ranking, $\hat \phi_w$, satisfying \eqref{eq:R}, on the collected data,  $w$. Generally speaking, for a fixed number of alternatives, the more alternative pairs which have been queried, the more informative we expect the ranking, $\hat \phi_{w}$. That is, there is a tradeoff between the amount of pairwise data collected and the informativeness  of the ranking. 
In this paper, we consider the following question:
 Given a  pairwise comparison dataset, $(w_{0},y_{0})$, and the opportunity to collect $\xi$ additional pairwise comparisons, which pairs  should be targeted to maximally improve the informativeness of a statistical ranking, $\hat \phi_w$, satisfying  \eqref{eq:R}? 
 
 We propose a learning algorithm for ranking from cardinal pairwise comparisons. 
To accomplish this, we follow the methodology of the optimal design community \citep{Haber:2008,Pukelsheim:2006,Melas:2006,Fedorov:1972}, and consider the \emph{Fisher information} for the ranking estimate, $\hat \phi_w$, denoted $\text{F.I.}(\hat{\phi}_{w})$.  We are thus led to the following bilevel optimization problem:
\begin{subequations}
\label{eq:bilevelOpt}
\begin{align}
\label{eq:bilevelOpta}
\max_{w} \ \ & f\left( \text{F.I.}(\hat{\phi}_{w}) \right) \\
\label{eq:bilevelOptb}
\text{ such that} \ \ & \hat \phi_{w} = \mathcal{R}(y,w) \\
\label{eq:bilevelOptc}
& w \in \mathbb Z^{N}_{+}, \ \ w \succeq w_{0},  \ \  \| w-w_{0} \|_{1} \leq \xi.
\end{align}
\end{subequations}
where $N:=\binom{n}{2}=\frac{n(n-1)}{2}$ and $f\colon \mathbb S_{+}^{n} \rightarrow \mathbb R$ is a convex function. For general optimal design problems, common choices for the scalar function $f(A)$  include
\begin{subequations}
\label{eq:otherOptCond}
\begin{align}
\label{eq:otherOptConda}
f(A) &= \min_{i} \ \lambda_{i}(A) &&\text{E-optimal} \\
\label{eq:otherOptCondb}
f(A) &= -\text{tr} A^{-1}  = - \sum_{i} \lambda_{i}(A^{-1}) &&\text{A-optimal} \\
\label{eq:otherOptCondc}
f(A) &= \det A = \prod_{i} \lambda_{i}(A)  &&\text{D-optimal}
\end{align}
\end{subequations}
where $\{\lambda_{i}(A)\}_{i=1}^{n}$ denote the eigenvalues of $A$. The constraint in \eqref{eq:bilevelOptc} specifies that  only a limited amount of additional data is collected. 

The ranking problem can be represented on a complete directed graph, $G=(V,A)$, with vertices representing the alternatives and the pairwise comparison data, $y$, is a function on the arcs. 
The queried pairs, $w$, can be viewed as an integer valued function on the arcs representing the number of times a pairwise comparison has been queried for that particular pair. In \S  \ref{sec:OSD}, we show that for the least squares ranking estimate, 
$\hat \phi_w  =  \arg \min_{\langle \phi, 1\rangle=0} \| B \phi - y \|_{2,w} $, where $B$ is defined in \S \ref{sec:AlgConn}, 
the constraint \eqref{eq:bilevelOptb} in the optimization problem \eqref{eq:bilevelOpt} decouples, yielding a graph synthesis problem of finding the graph whose $w$-weighted graph Laplacian has desired spectral properties. For example, an E-optimal design \eqref{eq:otherOptConda} corresponds to finding edge weights $w$ for which the $w$-weighted graph Laplacian has maximal second eigenvalue (algebraic connectivity). This reduction of the data collection problem to graph-theoretic questions is one of the primary contributions of this and previous work \citep{OBO2012}. 

 For the active learning problem for  ranking from \emph{ordinal} pairwise data, there has been a large amount of recent work,  which we briefly discuss in \S \ref{sec:relwork}. However,  the analogous cardinal problem considered here has received  less attention. Several recent papers have proposed using iid random sampling (corresponding to an Erd\"os-R\'enyi graph) for quality assessment algorithms and crowdsourcing experiments, see, \eg,   \citet{eichhorn2010} and \citet{xu2012}. 
 These algorithms collect  pairwise comparisons from a large number of distributed sources without considering the informativeness of the resulting rankings.  Like random sampling, the data collection  methodology advocated here does not depend on the previous pairwise preferences to select new pairwise queries; our proposed learning algorithm is  parallelizable. 
   
In \S  \ref{sec:numexp}, we consider several applications of the methodology developed in \S \ref{sec:OSD} for the optimal data collection problem \eqref{eq:bilevelOpt}. We begin with a few constructed examples and show that  graphs can be generated which have larger algebraic connectivity than Erd\"os-R\'enyi randomly generated graphs. The rankings of the datasets represented by these well-connected graphs are more informative then those represented by Erd\"os-R\'enyi graphs. We then consider the data collection problem for ranking Yahoo! movies and for the 2011-2012 NCAA Division 1 football season. 

\paragraph{Application: improving the informativeness of Yahoo! movie rankings} 
The Yahoo!  Movie user rating dataset consists of an incomplete user-movie matrix where entries represent a score given to the movie by the user. By considering the differences in movie reviews by each user, a pairwise comparison dataset $(w_0,y_0)$ can be constructed. For this dataset, we empirically demonstrate that the assumptions made in \S  \ref{sec:OSD} are reasonable.   By applying the methodology developed in \S  \ref{sec:OSD}, we show that the addition of a small number of well-chosen pairwise comparisons can significantly increase the Fisher informativeness of the ranking. 
The same number of randomly chosen additional pairs has no appreciable impact on the Fisher information.

\paragraph{Application: sports scheduling}  
The statistical ranking problem arises in competitive sports. Here, teams (alternatives) are ranked based on the schedule (queried alternative pairs) and the game results (pairwise comparisons). The dataset is incomplete if not all teams play all other teams;   inconsistent if there are teams A, B, and C, such that  team A beats team B, team B beats team C, and team C beats team A; and imbalanced if the ``strength of schedule'' varies among the teams.  
In this setting, the tradeoff between the amount of data collected (number of games) and the informativeness of the ranking is especially transparent. 
In a single elimination tournament with $n$ teams, there are only $n-1$ games played. Here, we expect that the ``best team'' wins the tournament, but it is difficult to rank the remaining teams in any reasonable way. At the other extreme, a round-robin tournament among  $n$ teams requires $\binom{n}{2}$ games which may not be possible if $n$ is large. 
The optimal data collection problem \eqref{eq:bilevelOpt} can be interpreted as designing the schedule so that the  rankings are the most informative, and thus we refer to the optimal design problem in this context as \emph{schedule design}. 
In \S \ref{sec:football}, we study the 2011-12 NCAA football schedule and, using the methodology developed in \S \ref{sec:OSD},  propose schedules with the same number of games which are significantly more informative. Using spectral clustering methods to identify highly-connected communities within the division, we argue that the NCAA could improve its notoriously poor rankings by simply scheduling more out-of-conference games. In \S \ref{sec:sythExp}, we continue with the graph constructed in \S \ref{sec:football} and demonstrate using synthetic data that ranking estimates obtained via active sampling are more accurate (in the sense of both the $L^2$-distance and the Kendall-$\tau$ rank distance) than via random sampling. 

\paragraph{Outline} 
In \S \ref{sec:relwork}, we review related work. 
In \S \ref{sec:AlgConn}, we review properties of the eigenvalues of the graph Laplacian and establish notation used in subsequent sections.
In \S \ref{sec:OSD} we study the optimal data collection problem  \eqref{eq:bilevelOpt} and show the reduction of \eqref{eq:bilevelOpt} to a graph synthesis problem. 
In \S \ref{sec:numexp}, we conduct a number of numerical experiments to demonstrate how the optimal data collection methodology developed in \S \ref{sec:OSD} can be employed. 
Finally, we conclude in \S \ref{sec:disc} with a discussion of further directions.

\section{Related work} \label{sec:relwork}
Our work is related to several subject areas, which we discuss in turn: active learning methods for ordinal ranking, statistics and experimental design, sports scheduling, and graph theory. 
This work is an extension of the conference proceeding, \citet{OBO2012}. 
 In particular, the present article 
includes a more extended survey of related work, 
provides a comparison of Erd\"os-R\'enyi graphs and those with maximal algebraic connectivity and a discussion of the implications of this for the optimal data collection problem, 
a more complete discussion of the scalarizing criterion  for the Fisher information \eqref{eq:otherOptCond},  
and  additional examples.

\paragraph{Active learning methods for ordinal ranking} 
 \citet{ailon2011} and \citet{Ailon2012} study the problem of optimally sampling preference labels for the minimum feedback arc-set in weighted tournaments (MFAST) also known as Kemeny-Young ranking. In this work, the dataset considered is ordinal, \ie, only pairwise preference labels are specified, whereas in the present work, the dataset is cardinal, \ie, the preferences are represented as quantitative (real valued) differences between items.   \citet{JamiesonNowakNIPS11} and \citet{Jamieson2012} consider the problem of actively learning the optimal permutation for a collection of alternatives under the assumption that the alternatives have additional geometric structure, namely a Euclidian embedding in a low dimensional space.  \citet{wauthier2013} propose ranking methods based on independent random sampling, which have worse theoretical complexity, but are relatively simple and  easily parallelizable.

\paragraph{Statistics and experimental design}
Excellent surveys of the optimal experiment design literature can be found  in  \citep{Haber:2008,Pukelsheim:2006,Melas:2006,Fedorov:1972}. Methods of optimal experiment design have been applied to ill-posed inverse problems, \eg, in geophysical \citep{Haber:2008} or biomedical imaging \cite[ch. 13, p. 273-290]{HHT2011}, \citep{ChungHaber12,quinnKeoughm02,DiStefano76}. 
It is instructive  to consider the analogy between these applications and the optimal data collection problem considered here. 
In imaging systems, there is a tradeoff between the amount of collected data and the accuracy of the reconstruction, or equivalently, the sparsity of the measurement and the uncertainty in the solution to the inverse problem. 
For application dependent reasons (\eg, high radiation dose to a patient or the cost of collecting data), it is often desirable to place as few sensors as possible while still maintaining an acceptable accuracy in the reconstruction. In the current work, the goal is to construct the best ranking possible from a small number of pairwise comparisons. In both situations, it is desirable to take ``measurements'' which are maximally informative.

\paragraph{Methods for scheduling from sports} As discussed in the introduction, in the context of sports ranking, \eqref{eq:bilevelOpt} is equivalent to optimal schedule design. There are large variations in the methods currently used for sports scheduling.  It is convenient to distinguish between \emph{static} and \emph{dynamic} scheduling. In static scheduling, the schedule is determined prior to the season, independent of the performance of teams throughout the season. Examples of leagues employing static schedules include NCAA  football and Major League Baseball (MLB). In dynamic scheduling, the schedule is determined by past score results. For example, in a single elimination tournament, a team advances to the next round only if they win in the current round. Leagues which partially rely on single elimination tournaments include ATP tennis and  FIFA World Cup soccer. 
\citet{Glickman:2005} proposes a dynamic scheduling method where games are scheduled which maximize the expected gain in information and thus one can view the resulting schedules as a greedy algorithm to learn as much as possible about the rankings. This active learning algorithm is similar to several in the machine learning community, where past observations are used to control the process of gathering future observations, see, for example, \citep{krause2008,seeger2011,SilvaCarin2012}.
While dynamic schedules utilize the results of previous games and can thus be more informative than static schedules, they have the disadvantage that they may not be completely determined prior to the season. The algorithm developed in this paper is a static scheduling method. 

Another type of scheduling in sports focuses on the seeding policy of single-elimination  tournaments 
with the objective of arranging the teams so that the outcome of the tournament agrees with a preexisting ranking  \citep{Glickman2008,DSouza:2010fk,Scarf:2011fk} or an arrangement which favors a particular team \citep{Vu:2001uq}. These objectives depend on a preexisting ranking of the teams, which we do not assume to know in this paper. Another type of tournament scheme is investigated by \citet{Ben-Naim:2007}, where a sequence of rounds of diminishing size are used to determine the best team.

\paragraph{Graph theory}
In this paper, we reduce the schedule design problem \eqref{eq:bilevelOpt} to a graph synthesis problem. 
We focus on the optimality condition given in \eqref{eq:otherOptConda}, which reduces to finding graphs with maximal algebraic connectivity. 
There is a tremendous amount of work on the algebraic connectivity of graphs, originating with studies by Miroslav Fiedler  \citep{Fiedler:1973}. 
Many properties of algebraic connectivity are reviewed in \citep{Mohar:1991,Biyikoglu:2007} and we also review some of these results in \S \ref{sec:AlgConn}. 
The problems arising from the other  optimality conditions, \eqref{eq:otherOptCondb} and \eqref{eq:otherOptCondc}, are less well studied \citep{Grimmett:2010,Ghosh:2006,Ghosh:2008}.

The robustness of a network to node/edge failures is highly dependent on the algebraic connectivity of the graph. 
Also, the rate of convergence of a Markov process on a graph to the uniform distribution is determined by the algebraic connectivity  \citep{boydDiaconisXiao2004,Sun2004}. 
In the ``chip-firing game'' of \citet{Bjorner:1991}, the algebraic connectivity dictates the length of a terminating game. Consequently, algebraic connectivity is a measure of performance for the convergence rate in sensor networks, data fusion, load balancing, and consensus problems \citep{Olfati-Saber:2007}. 

Finally, we mention recent work of \citet{boumal2012} on a problem of estimating a set of rotations from a set of noisy measurements. Here, bounds on synchronization are connected to the algebraic connectivity  of a measurement graph, where the edge weights are proportional to the measurement quality. 

\section{Eigenvalues of the graph Laplacian and the algebraic connectivity} \label{sec:AlgConn}
In this section, we briefly survey relevant results on the eigenvalues of the graph Laplacian and algebraic connectivity. More extensive treatments are given in \citep{Fiedler:1973,Biyikoglu:2007,Mohar:1991,Chung:1997}. In \S \ref{sec:compGraphs}, we recall algorithms for computing graphs with large algebraic connectivity \citep{Ghosh:2006b,Wang:2008}. 

Let  $B\in \mathbb R^{N\times n}$ where $N:= \binom{n}{2}$ be the arc-vertex incidence matrix for the complete directed graph $G=(V,E)$ on $|V| = n$ nodes,
\begin{equation}
\label{eq:incidence}
B_{k,j} = \begin{cases}  
1 & j = \text{head} (k) \\
-1 &j = \text{tail} (k)  \\
0 & \text{otherwise}.
\end{cases}
\end{equation}
Here, we have used the terminology that if an arc $k=(i,j)$ is directed from node $i$ to node $j$ then $i$ is the \emph{tail} and $j$ is the \emph{head} of  arc $k$. The arc orientations (heads and tails of arcs) can be chosen arbitrarily. The matrix $B$ as defined in \eqref{eq:incidence} is also sometimes referred to as the graph gradient \citep{Hirani:2011fk,Jiang2010}. 
Given an edge-weight $w \in \mathbb Z^N_+$, the $w$-weighted graph Laplacian is defined 
$$
\Delta_{w} := B^{t} W B \qquad \text{where } W = \text{diag}(w).
$$  
If we consider a subset of the edges, $\tilde E \subset E$, and let $w$ be the indicator function on $\tilde E$, then  $\Delta_{w}$ is referred to as the un-normalized (symmetric)  graph Laplacian for $(G,\tilde E)$. One may interpret the triple $(G,E,w)$ for $w \in \mathbb Z^N_+$ as a directed multigraph where $w_k$ for $k=(i,j)$ is the number of arcs connecting vertices $i$ and $j$. 
The $w$-weighted degree vector $d\in \mathbb R^n$ is defined by $d_i = \sum_j w_{ij}$. Let $M:=\| w\|_1 = \frac{1}{2} \|d\|_1$ and $d_+$ and $d_-$ denote the maximum and minimum $w$-weighted degrees in the graph.

Let $\lambda_{i}(w)$ for $i=1,\ldots,n$ denote the eigenvalues of the $w$-weighted graph Laplacian, $\Delta_{w}$. 
The eigenvalues are contained in the interval $[0, d_{+}]$. The first eigenvalue of $\Delta_{w}$, $\lambda_{1}$, is zero with corresponding eigenvector $v_{1}=1$. The second eigenvalue, $\lambda_{2}$, is nonzero if and only if the graph is connected. The second eigenvalue is characterized by
\begin{equation}
\label{eq:lam2}
\lambda_{2}(w) = \min_{\substack{ \|v\|=1 \\ \langle v, 1 \rangle = 0}} \| B v\|_{2,w}.  
\end{equation}
In the case where $w$ is the indicator function for an edge set $\tilde E$,  $\lambda_{2}(w)$ is referred to as the \emph{algebraic connectivity} of the graph $G=(V,\tilde E)$. The eigenvector $v_{2}$ corresponding to $\lambda_{2}$ is sometimes called the Fiedler vector after Miroslav Fiedler for his contribution to the subject \citep{Fiedler:1973}. For $w\in \mathbb Z_+^N$, $\lambda_{2}(w)$ is the algebraic connectivity for the multigraph with $w_{ij}$ edges between nodes $i$ and $j$.

Let $w_{i}\in \mathbb Z_+^{N}$ for $i=1,2$ be edge weights on $G$. It follows from \eqref{eq:lam2} that $w_{1} \leq  w_{2}$ implies $\lambda_2(w_{1})\leq \lambda_{2}(w_{2})$. That is, the function $\lambda_{2}(w)$ is non-decreasing in $w$. In particular, if $w_{i} \in \{0,1\}^{N}$ are the indicator functions for  two  edge sets $E_i$, $i=1,2$ and $w_{1} \leq  w_{2}$ (component-wise), then $E_{1}\subseteq E_{2}$ and the more connected graph has greater algebraic connectivity.

Let $U \subset V$ and $\text{cut}(U,U^{c}) := \sum_{i\in U, j\in U^c} w_{ij}$ measure the set of edges connecting $U$ and $U^{c} := V\setminus U$. 
Then the algebraic connectivity is bounded by the normalized graph cut, 
\begin{equation}
\label{eq:edgeCut}
\lambda_{2}(w) \leq \min_{U\subseteq V} \frac{n | \text{cut}(U,U^{c})|}{|U| |U^{c}|}.
\end{equation}
In particular, if $U = \{v\}$ where $v\in V$ is the node with smallest degree, i.e., $d_{v} = d_{-}$, then $d_{v} \leq \frac{2M}{n}$  where $M = \| w\|_1$ and  we obtain
\begin{equation}
\label{eq:lam2ub}
\lambda_{2}(w) \leq \frac{n d_{-}}{ n-1} \leq \frac{2 M}{n-1}.
\end{equation}
Properties of graphs for which the bound in \eqref{eq:lam2ub} is tight have been studied \citep{Fallat:2003}.

If $w \in \{0,1\}^{N}$ is the indicator function for an incomplete edge set $\tilde E$ and $\tilde G := (V,\tilde E)$, then the edge connectivity of a $\tilde G$, $C_e(\tilde G)$, is the minimal number of edges whose removal would result in a disconnected graph, 
$$
C_e(\tilde G) = \min_{A \subset V} \ \sum_{i\in A, j\in A^c} w_{ij}. 
$$ 
The vertex connectivity of $\tilde G$, $C_v(\tilde G)$ is the minimal number of vertices (together with adjacent edges) whose removal would result in a disconnected graph. In this case,  the algebraic connectivity is bounded above by both the edge and vertex connectivities,  
$$
\lambda_{2}(w) \leq C_{v}(\tilde G) \leq C_{e}(\tilde G),
$$
\citep{Fiedler:1973}.  The algebraic connectivity can also be bounded in terms of Cheeger's inequality, Buser's inequality, and the diameter of the graph \citep{Biyikoglu:2007,Mohar:1991,Chung:1997}. 

There are also a number of results for the perturbation of the eigenvalues of $\Delta_w$ under changes to the weights $w$. 
Let $\Delta_{w} v = \lambda v$, $\lambda>0$, $w_{0} = \min_{k} w_{k}$ and $w' = w-w_{0}$. Then 
\begin{equation}
\label{eq:fullGraph}
\Delta_{w'} v = (\lambda - w_{0}n) v.
\end{equation}
This follows from the fact that $B^{t}B = n \ \text{Id} - 1_{n} 1_{n}^{t}$. Thus, adding weight $w_0$ to $w$ simply increases all of the eigenvalues of $\Delta_w$ by $w_0$. 

Consider the weight $w' = w + \delta_{k}$ where $\delta_{k}$ is the indicator function for edge $k$. Then using Weyl's theorem \citep{Horn1990}, we obtain
\begin{align}
\label{eq:maxIncrease}
\lambda_{2}(w') &\leq \lambda_{2}(w) + \| B^{t} \text{diag}(\delta_{k}) B\| 
=   \lambda_{2}(w) + 2.
\end{align}

Consider the weight $w' = w + \delta_{k}$ where $\delta_{k}$ is the indicator function for edge $k$. 
Denote the eigenvalues of the $w$ and $w'$-weighted graph Laplacians by $\lambda_j$ and $\lambda_j'$ respectively. Then the eigenvalues $\lambda$ and $\lambda'$ interlace \citep{Mohar:1991}, \ie, 
\begin{equation}
\label{eq:interlace}
0 = \lambda_1 = \lambda_1' \leq \lambda_2 \leq \lambda_2' \leq \lambda_3 \leq \ldots \leq \lambda_n \leq \lambda_n'. 
\end{equation}

\subsection{Finding graphs with large algebraic connectivity} \label{sec:compGraphs}
In several applications, it is useful to compute graphs with large algebraic connectivity, \eqref{eq:lam2}. 
The problem of finding weights $w\in \mathbb R^{N}$ which maximize $\lambda_{2}(w)$ is a convex optimization problem and can be formulated as a semidefinite program (SDP)  \citep{Ghosh:2006b}. However, if $w\in \mathbb Z_{+}^{N}$, the problem is NP-hard  \citep{Mosk-Aoyama:2008}. This is the case arising in the optimal data collection problem. 

The integer constrained problem may be solved by relaxing to the unconstrained problem and then rounding the solution. This is clearly a lower bound on the optimal solution and, if the values $w$ are large, a reasonable approximation. 
Another approach,  advocated by \citep{Ghosh:2006b,Wang:2008}, is to use the greedy algorithm based on the Fiedler vector described in  Algorithm \ref{alg:Greedy}. This algorithm adds a specified number of edges to an input graph to maximize the algebraic connectivity of the resulting augmented graph. In this work, we refer to graphs produced via this method as \emph{nearly-optimal}. 

\begin{algorithm}[t]
\caption{\label{alg:Greedy} A greedy heuristic for finding integer-valued  edge weights $w$ for which the $w$-weighted graph Laplacian has large second eigenvalue   \citep{Ghosh:2006b,Wang:2008}. See \S \ref{sec:compGraphs}. }
\vspace{.2cm}

\begin{algorithmic}
\STATE{\bfseries Input:} An initial edge weight $w_{0}\in \mathbb Z_+^N$ defined on the complete graph of $n$ nodes and an integer, $\xi$.  \\

\vspace{.2cm}

\STATE{\bfseries Output:} An edge weight, $w\succeq w_0$, such that  $\|w-w_{0}\|_{1} =  \xi$, and $\Delta_{w}$ has large second eigenvalue. \\

\vspace{.2cm}

\STATE Set $w = w_{0}$  (current edge weight)
\FOR {$ \ell = 1$ \ {\bfseries to} \ $\xi$,}
\STATE Compute the second eigenvector, $\displaystyle F = \arg \min_{\substack{ \|v\|=1 \\ \langle v, 1 \rangle = 0}} \| B v\|_{w}$
\STATE Find the edge $(i,j)$ which maximizes $(F_{i} - F_{j})^{2}$
\STATE Set $w = w + \delta_{ij}$
\ENDFOR
\end{algorithmic}
\end{algorithm}

\section{Optimal scheduling using a least squares ranking } \label{sec:OSD}
We assume that each alternative $j=1,\ldots, n$ has a ranking (measure of strength) given by $\phi_{j}$. We consider a complete graph with $n$ nodes representing the alternatives. The edges of the graph are given an arbitrary orientation and enumerated $k = 1,\ldots, \binom{n}{2}\equiv N$. Let $B\in \mathbb R^{N\times n}$ denote the arc-vertex incidence matrix \eqref{eq:incidence} for the complete graph. For each ordered pair $k = (i,j)$, we assume that the pairwise comparison data collected is of the form 
\begin{equation}
\label{eq:errorModel}
y_{k} = (B\phi)_{k} + \epsilon_{k} 
\end{equation} 
where $\epsilon_{k}$ is a random variable with zero mean, \ie,  $\E \epsilon = 0$. 
Let $w_{k} \in \mathbb Z_{+}$ denote the number of pairwise comparisons between  alternatives  $i$ and $j$. 
We assume that the variance of $\epsilon_{k}$ is given by $\sigma^{2}/w_{k}$ for some constant $\sigma$. More comparisons between alternatives $i$ and $j$, reduce the variance in the observed pairwise comparisons.

\paragraph{Ranking} There are several choices for the ranking $\mathcal{R} (y,w)$ in \eqref{eq:R}.  The Gauss-Markov theorem states that the least squares estimator, 
\begin{subequations}
\label{eq:LSE}
\begin{align}
\hat{\phi}_{w} 
\label{eq:LSEb}
&= \arg \min_{\langle \phi, 1\rangle=0} \| B \phi - y \|_{2,w} \\
\label{eq:LSEc}
&= (B^{t} W B)^{\dag}B^{t} W y,
\end{align}
\end{subequations}
is the linear, unbiased ($\mathbb E [\hat{\phi}_{w}] = \phi$) estimator with smallest covariance. In \eqref{eq:LSEc}, $W := \text{diag} (w)$ is the diagonal matrix with entries $W_{kk} = w_k$. 
Equation \eqref{eq:LSEb} can be interpreted as finding a potential function, $\phi$, defined on the vertices, such that the gradient of the potential function agrees with the pairwise comparisons in the least squares sense. The least-squares estimate  \eqref{eq:LSE}  is also sometimes referred to as  HodgeRank \citep{Jiang2010} and is related to the Massey and Colley methods used in sports rankings \citep{Langville:2012}.  
The least squares estimator has proven to have relatively good predictive power when empirically compared against a number of other ranking methods on sports datasets \citep{Barrow2013}  and is the ranking method considered in the present work.

\begin{prop}
\label{prop:FI}
Consider the data model \eqref{eq:errorModel} where $\epsilon$ is a random vector with $\E \epsilon = 0$ and $\mathrm{Var} (\epsilon) = \sigma^{2} W^{-1}$ where $W=\mathrm{diag}(w)$ and $w\in\mathbb Z_{+}^{N}$. The Fisher information of the least squares estimator $\hat{\phi}_{w} $, as defined in \eqref{eq:LSE}, is given by 
\begin{equation}
\label{eq:FI}
\mathrm{F.I.}(\hat{\phi}_{w}) = \sigma^{-2} ( B^{t} W B) = \sigma^{-2} \Delta_w,
\end{equation}
where $\Delta_w$ is the $w$-weighted graph Laplacian. 
\end{prop}
\begin{proof}
Let $\hat{\phi}_{w}$ be the least squares estimator  \eqref{eq:LSE} for $\phi$ in \eqref{eq:errorModel}. 
We first compute
\begin{align*}
\hat{\phi}_{w} &= (B^{t} W B)^{\dag} B^{t} W y = (B^{t} W B)^{\dag} B^{t} W (B \phi + \epsilon) = \phi +  (B^{t} W B)^{\dag} B^{t} W \epsilon.
\end{align*}
Thus,
\begin{align*}
\text{Var} (\hat{\phi}_{w} ) &= \E \left[ (\hat{\phi}_{w}  - \phi) (\hat{\phi}_{w} -\phi)^{t} \right] =(B^{t} W B)^{\dag} B^{t} W  \E \left[ \epsilon \epsilon^{t} \right] W B  (B^{t} W B)^{\dag}.
\end{align*}
Assuming that $\E \left[ \epsilon \epsilon^{t} \right] = \sigma^{2} W^{-1}$, we obtain 
\begin{equation}
\label{eq:Var}
\text{Var}(\hat{\phi}_{w}) = \sigma^2 ( B^{t} W B)^{\dag} = \sigma^2 \Delta_w^{\dag},
\end{equation}
which is the Moore-Penrose pseudoinverse of the $w$-weighted graph Laplacian. 
Since the least squares ranking is unbiased, \ie, $\E \hat \phi_{w} = \phi$, the Fisher information is the pseudoinverse of the covariance matrix, $\text{Var}( \hat \phi_{w})$. 
 \end{proof}

\paragraph{Optimal Data Collection}
The optimal data collection problem \eqref{eq:bilevelOpt} is a scalarization of maximizing $\text{F.I.}(\hat \phi_w)$ in the sense of the semi-definite ordering (\ie, $A \geq B$ if $A-B \succeq 0$). 
Traditional optimality criteria are functions of the eigenvalues of $\text{F.I.}(\hat{\phi}_{w})$ such as given in \eqref{eq:otherOptCond} \citep{Haber:2008,Pukelsheim:2006,Melas:2006,Fedorov:1972}. 
\begin{prop}
\label{prop:minLam2}
Consider the data model \eqref{eq:errorModel} with $\epsilon$ as in Prop. \ref{prop:FI} and 
let $\hat{\phi}_{w}$ be the least squares estimator \eqref{eq:LSE}. The three optimality criteria \eqref{eq:otherOptCond}  for the bi-level optimization problem \eqref{eq:bilevelOpt} are given by 
\begin{subequations}
\label{eq:FIrank}
\begin{align}
\label{eq:FIranka}
f\left(\mathrm{F.I.}(\hat{\phi}_{w}) \right) = \lambda_2(w)  &&\text{E-optimal} \\
\label{eq:FIrankb}
f\left(\mathrm{F.I.}(\hat{\phi}_{w}) \right) = - \sum_{i\geq2} \lambda^{-1}_{i}(w)  &&\text{A-optimal} \\
\label{eq:FIrankc}
f\left(\mathrm{F.I.}(\hat{\phi}_{w}) \right) = \prod_{i\geq2} \lambda_{i}(w)  &&\text{D-optimal},
\end{align}
\end{subequations}
where  $\lambda_{i}(w)$ for $i=1,\ldots,n$ denote the eigenvalues of the $w$-weighted graph Laplacian, $\Delta_{w}$. 
\end{prop}
\begin{proof}
For a connected graph, the only zero eigenvalue of the graph Laplacian is the first one. The expressions in \eqref{eq:FIrank} then follow directly from  $\mathrm{F.I.}(\hat{\phi}_{w}) = \Delta_w$, as shown in Prop. \ref{prop:FI}, and the optimal criteria definitions in \eqref{eq:otherOptCond}. 
\end{proof}
Proposition \ref{prop:FI} shows that $\text{F.I.} (\hat{\phi}_{w})$ doesn't depend on the scores, $y$. Consequently, the constraint in the optimal data collection problem \eqref{eq:bilevelOptb} decouples. Using the E-optimal criteria \eqref{eq:FIranka}, the bilevel optimization problem \eqref{eq:bilevelOpt} reduces to the following eigenvalue optimization problem 
\begin{align}
\label{eq:minLam2}
\max_{w} \ \ & \lambda_{2}(w)\\
\nonumber
\text{ such that} \ \ & w \in \mathbb Z^{N}_{+}, \ \ w \succeq w_{0},  \ \  \| w-w_{0} \|_{1} \leq \xi.
\end{align}
Equation \eqref{eq:minLam2} can be interpreted as the graph synthesis problem of adding $\xi$ edges to the multigraph representing the dataset to maximize the algebraic connectivity.

\begin{remark}
\label{rem:OtherOptCond}
The A- and D-optimal conditions given in Proposition \ref{prop:minLam2} also have interesting interpretations in terms of the graph. 
By Kirchhoff's matrix-tree theorem, the D-optimal condition can be interpreted as the number of spanning trees within the graph \citep{Ghosh:2006}.  
The A-optimal condition is the  total effective resistance of a electric circuit constructed by identifying each edge of the graph with a  resistor of equal resistance \citep{Ghosh:2006,Ghosh:2008} and is related to the return time for a reversible Markov chain \citep{Grimmett:2010}.
\end{remark}

We also comment that the T-optimality condition, $ \text{tr} \left(\mathrm{F.I.}(\hat{\phi}_{w}) \right)$, which is another criteria commonly used in optimal design, simplifies in this setting to $ \text{tr} (\Delta_w) = \| w \|_1$, which is simply the total number of pairwise comparisons.

\section{Numerical experiments} \label{sec:numexp}
In this section, we study graphs corresponding to datasets which have informative rankings, which, by Proposition \ref{prop:minLam2}, are those with large algebraic connectivity. 
In \S \ref{sec:simpGraphs}, we consider structured graphs for which the eigenvalues of the Laplacian can be analytically computed and  small graphs with $\leq 5$ edges. 
In \S \ref{sec:compAlgConn}, we compare the expected algebraic connectivity of Erd\"os-R\'enyi random graphs with graphs obtained using the greedy algorithm described in \S \ref{sec:compGraphs}. 
In \S \ref{sec:movies}, we consider the informativeness of the ranking for the  Yahoo! Movie user ratings dataset. 
In \S \ref{sec:football}, we discuss the algebraic connectivity for the graph corresponding to the 2011-12 NCAA Division I   football schedule.  In \S \ref{sec:sythExp}, we continue with the graph constructed in \S \ref{sec:football} and demonstrate using synthetic data that ranking estimates obtained via active sampling are more accurate (in an $L^2$ sense) than via random sampling.

 \begin{table}[t]
\begin{center}
{\footnotesize
\begin{tabular}{l | S S S  S}
&&&& complete\\
& path, $P_{n}$ &  cycle, $C_{n}$ & complete, $K_{n}$ & bipartite, $K_{n,\ell}$ \\
\hline
diagram  &
\begin{tikzpicture} 
[inner sep=.4mm,
dot/.style={circle,draw=blue!50,fill=blue!20,thick}] 
\node (1) at (0,1)[dot] {1};
\node (2) at (.8,1)[dot] {2}
	edge[,thick] (1);
\node (3) at (1.6,1)[dot]{3}
	edge[,thick]  (2);
\node (4) at (2.4,1)[dot]{4}
	edge[,thick] (3);
\end{tikzpicture}
&
\begin{tikzpicture} 
[inner sep=.4mm,
dot/.style={circle,draw=blue!50,fill=blue!20,thick}] 
\node (1) at (0,.8)[dot] {1};
\node (2) at (.8,.8)[dot] {2}
	edge[,thick] (1);
\node (3) at (.8,0)[dot]{3}
	edge[,thick]  (2);
\node (4) at (0,0)[dot]{4}
	edge[,thick] (3)
	edge[,thick] (1);
\end{tikzpicture}
&
\begin{tikzpicture} 
[inner sep=.4mm,
dot/.style={circle,draw=blue!50,fill=blue!20,thick}] 
\node (1) at (0,.8)[dot] {1};
\node (2) at (.8,.8)[dot] {2}
	edge[,thick] (1);
\node (3) at (.8,0)[dot]{3}
	edge[,thick]  (1)
	edge[,thick]  (2);	
\node (4) at (0,0)[dot]{4}
	edge[,thick] (1)
	edge[,thick] (2)
	edge[,thick] (3);
\end{tikzpicture}
& 
\begin{tikzpicture} 
[inner sep=.4mm,
dot/.style={circle,draw=blue!50,fill=blue!20,thick}] 
\node (1) at (0,1.2)[dot] {1};
\node (2) at (0,.4)[dot] {2};
\node (3) at (.8,1.6)[dot] {3}
	edge[,thick] (1)
	edge[,thick] (2);
\node (4) at (.8,.8)[dot] {4}
	edge[,thick] (1)
	edge[,thick] (2);
\node (5) at (.8,0)[dot] {5}
	edge[,thick] (1)
	edge[,thick] (2);
\end{tikzpicture}
\\
eigenvalues &$2-2\cos(\pi k /n)$ &$2-2\cos(2\pi k/n)$& $0_{1}$ , $n_{n-1}$ & $0_{1}$ , $n_{\ell-1}$, \\ 
&$k=0,\ldots n-1$ & $k=0,\ldots n-1$&&$\ell_{n-1}$, $(\ell+n)_{1}$  \\
\\ 
alg. conn. \eqref{eq:lam2} &$2-2\cos(\pi/n)$ & $2-2\cos(2 \pi/n)$ & $n$& $\min(n,\ell)$ \\
edge conn.  &1&2&$n-1$& $\min(n,\ell)$ \\
diameter & $n-1$ & $\lfloor n/2 \rfloor$ & 1 & 2
\end{tabular}
}
\end{center}
\caption{A comparison of several measures of connectivity for 4 well-known graphs. We assume $n\geq3$. Subscripts on the eigenvalues denote multiplicity and $\lfloor \cdot \rfloor$ indicates the floor function. See \S \ref{sec:simpGraphs}. }
\label{tab:PropGraphs}
\end{table}%

 \begin{figure}[t!]
\begin{center}
\includegraphics[width=.56\textwidth]{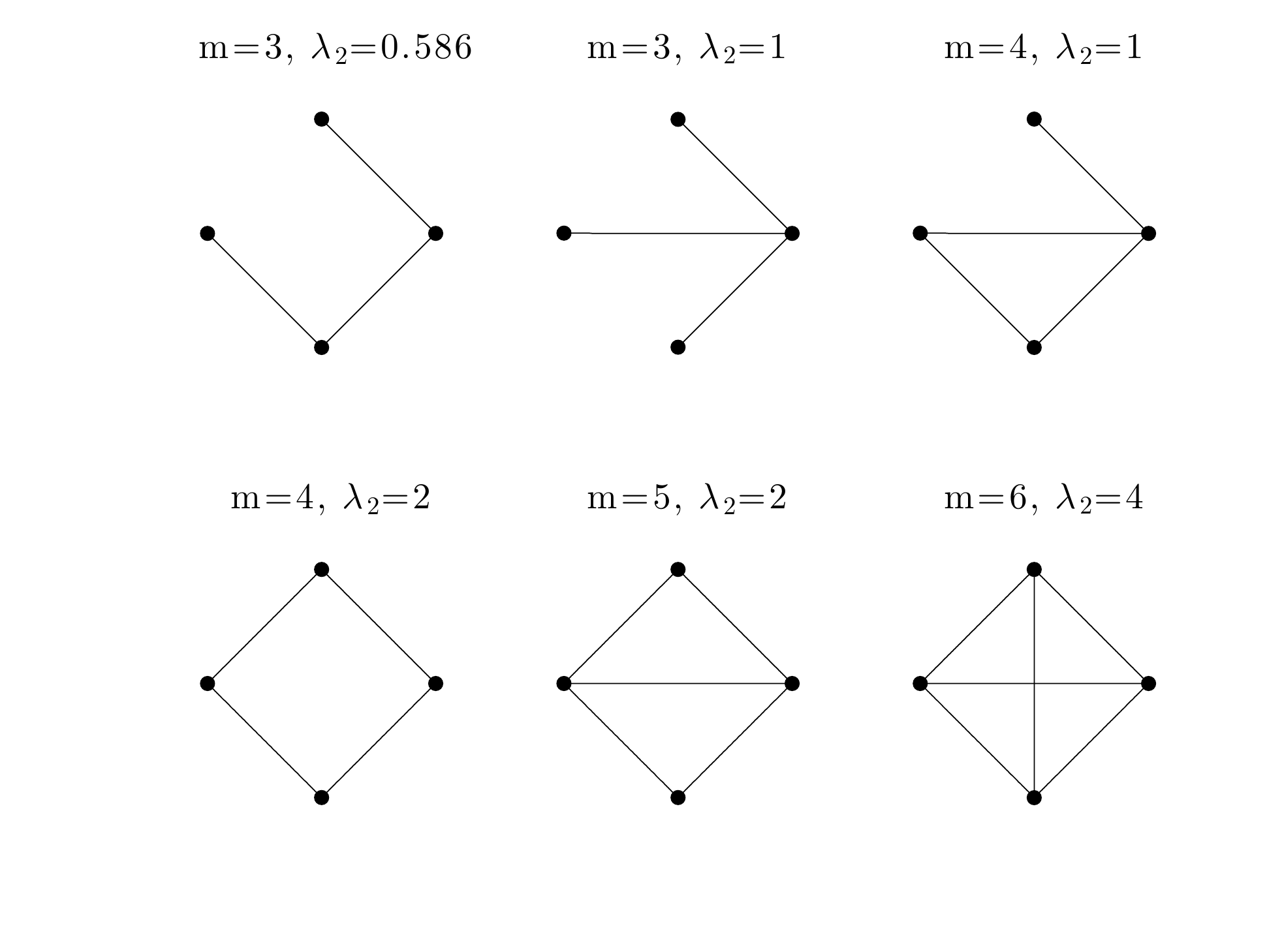} \\ 
\includegraphics[width=.78\textwidth]{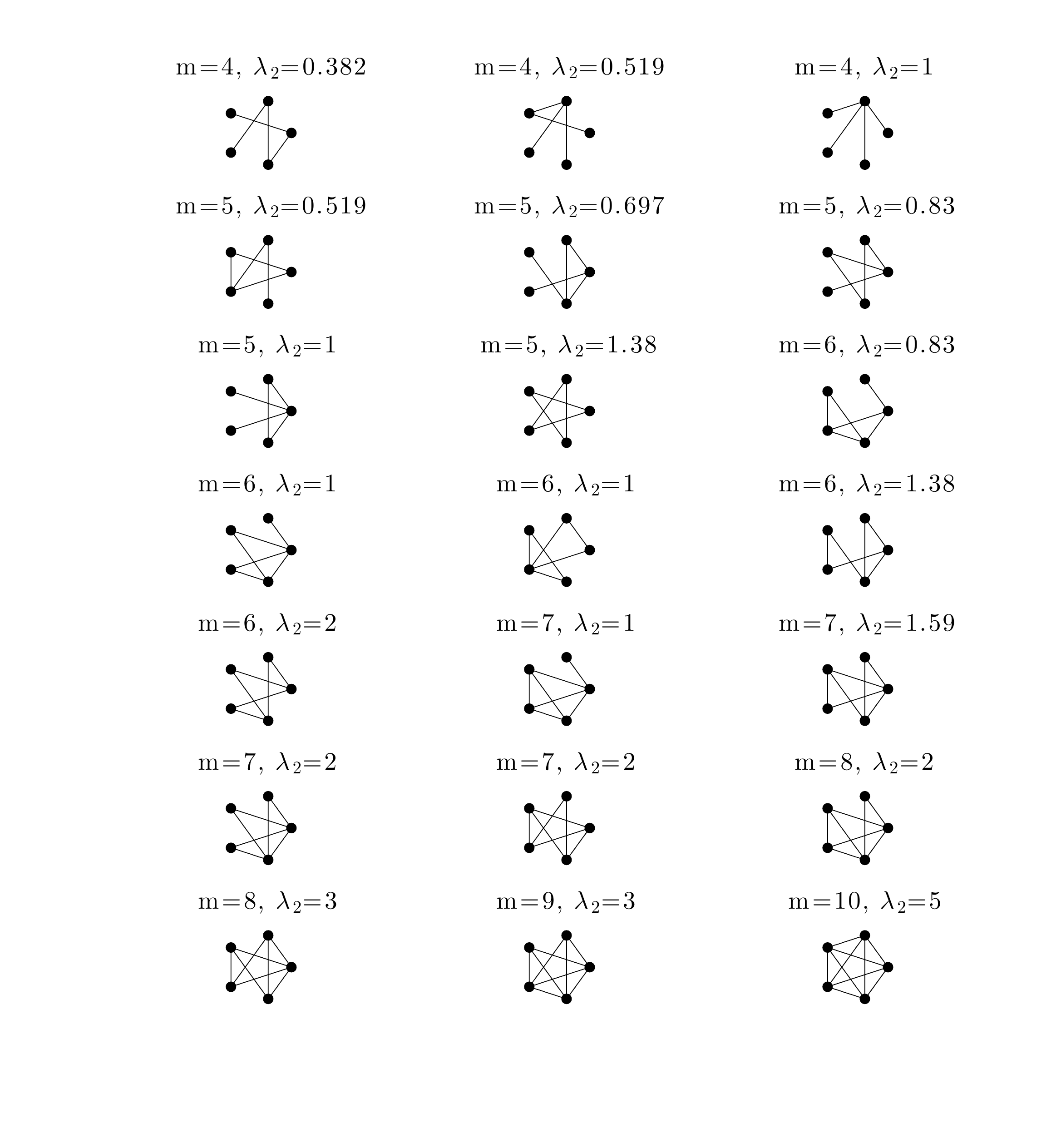} \\ \vspace{-1cm}
\caption{The $4$- and $5$-node connected graphs and their  algebraic connectivity, $\lambda_{2}$. Graphs with large algebraic connectivity represent datasets with informative rankings. See \S \ref{sec:simpGraphs}. \label{fig:simpGraphs}}
\end{center}
\end{figure}

\subsection{Algebraic connectivity for example  graphs} \label{sec:simpGraphs}
In this section, we give results on the algebraic connectivity for graphs with easily computable spectra and graphs with a small number of nodes. 
In Table \ref{tab:PropGraphs}, we tabulate the eigenvalues, algebraic connectivity \eqref{eq:lam2}, edge connectivity, vertex connectivity,  and diameter  for 4 well-known graphs.

The number of distinct $n$-node, connected, unlabeled graphs for $n=$1, 2, 3, \ldots  are  1, 1, 2, 6, 21, 112, 853, 11117, 261080,\ldots (Sloane A001349\footnote{The On-Line Encyclopedia of Integer Sequences, published electronically at \url{http://oeis.org}}). In Fig. \ref{fig:simpGraphs} we plot, for $n=4$ and $n=5$, each of these graphs together with the algebraic connectivity, $\lambda_{2}$. 
 In Fig. \ref{fig:simpGraphs}, we observe that as the number of edges, $m$, is increased, the algebraic connectivity, $\lambda_{2}$, generally increases. Furthermore, for a fixed number of edges,  $m$, the algebraic connectivity can  vary significantly. For $m=5$, $6$, and $7$, the value of $\lambda_{2}$ varies by a factor  $\geq2$. 
For $m=5$, the graph with smallest $\lambda_{2}$ has small edge connectivity (and hence small algebraic connectivity) and the graph with largest $\lambda_{2}$ has nodes with equal degree. 
 These small graphs beautifully illustrate the bounds given in \S \ref{sec:AlgConn}.
 
In Fig. \ref{fig:simpleExample}, we illustrate the effect of adding edges on the algebraic connectivity of a graph by studying \eqref{eq:minLam2} where $\| w_0 \|_1 =6$ and $\xi=1$. Although the graphs in Fig. \ref{fig:simpleExample} are small in size, it is already nontrivial to determine which edge should be added to maximally increase the algebraic connectivity. We observe that for graphs with low algebraic connectivity, a significant gain can be achieved, while the results for graphs with relatively high algebraic connectivity are modest. In the lowermost panel in Fig. \ref{fig:simpleExample}, the algebraic connectivity remains constant as an edge is added.  This follows from the fact that the second eigenvalue for the graph on the left has multiplicity 2 and the interlacing property described in \eqref{eq:interlace}. 

\begin{figure}[t!]%
	\begin{center}%
		\includegraphics[width=.3\textwidth]{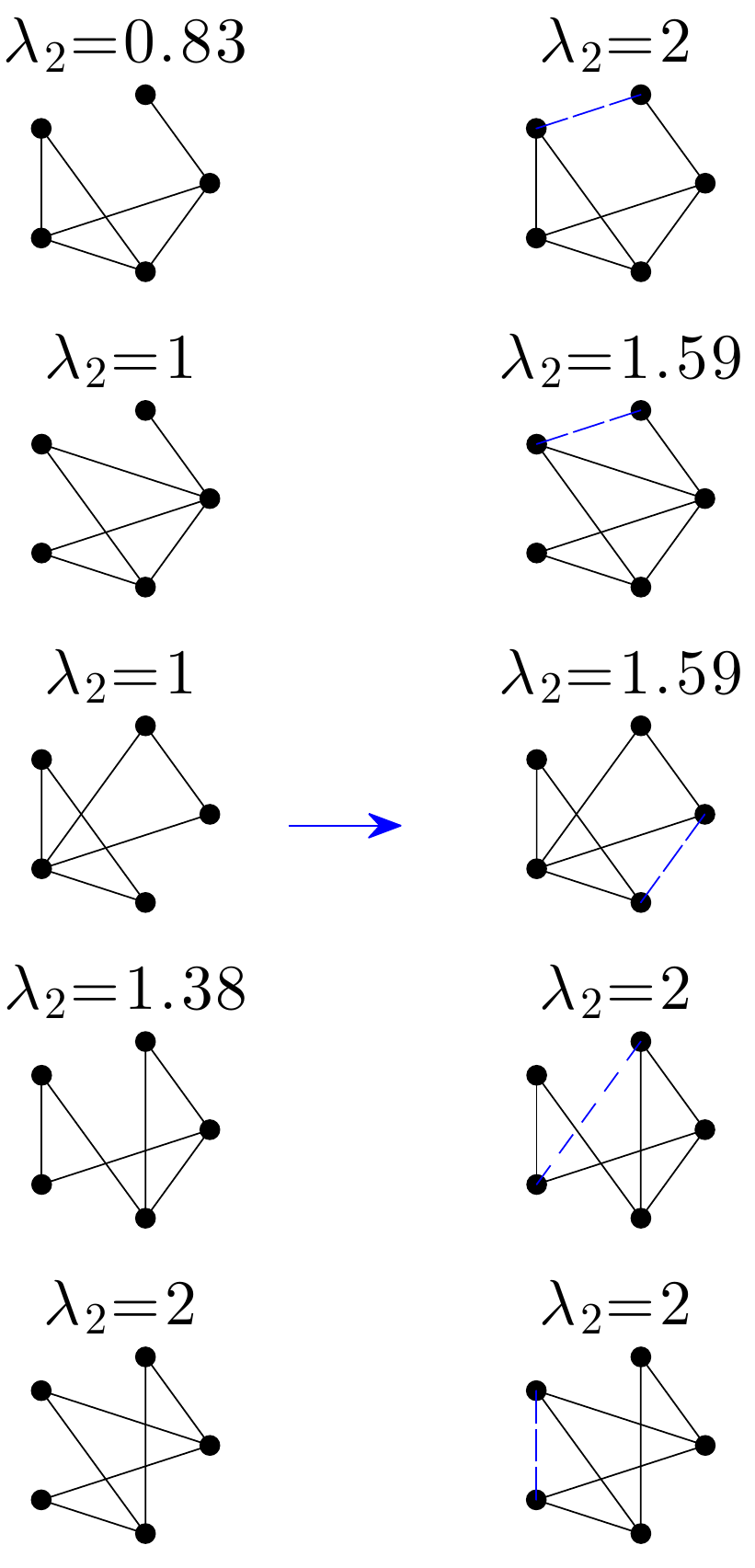}%
	\caption{Targeted data collection for small  graphs. 
	{\bf (left)} The five topologically distinct connected graphs with $n=5$ nodes and $m=6$ edges. {\bf (right)} For each edgeset on the left, we select one additional edge (blue dashes) so that $\lambda_2$ for the perturbed graph is maximal. The algebraic connectivity  of each graph is indicated. By Prop. \ref{prop:minLam2}, a ranking on a dataset represented by a graph on the right is more informative than one from a graph on the left. See \S \ref{sec:simpGraphs}.
 \label{fig:simpleExample}}%
	\end{center}%
\end{figure}%

Further consideration of the algebraic connectivity for certain families of graphs is considered in \citet{Kolokolnikov2013}. Here, it is observed that the greedy algorithm (Algorithm \ref{alg:Greedy}) is unable to discover certain small, structured graphs with maximal algebraic connectivity.

\subsection{Algebraic connectivity of Erd\"os-R\'enyi random graphs and computed nearly-optimal graphs} 
\label{sec:compAlgConn}

We consider the Erd\"os-R\'enyi random graph model $G(n,p)$ containing graphs with $n$ nodes and edges included with probability $p$, independent from every other edge. The expected number of edges for a graph in $G(n,p)$ is $p \binom{n}{2} $ and the threshold for connectedness is $p_{c} = \frac{\log n}{n}$.

There are several results on the spectrum of the graph Laplacian for Erd\"os-R\'enyi graphs, especially in the limit $n\uparrow \infty$; see, for example, \citep{juhasz1991,Chung2003,Feige:2005,coja2007,jamakovic2008,oliveira:2009,Chung:2011,Kolokolnikov2013b}. The algebraic connectivity of  Erd\"os-R\'enyi, Watts-Strogatz, and Barab\'asi-Albert random graphs has been studied numerically in \citep{Jamakovic:2007}.   The algebraic connectivity of a Watts-Strogatz graph is known to have a phase transition \citep{Olfati-Saber:2007}. 

We will utilize the following elementary upper bound on the algebraic connectivity, analogous to \eqref{eq:lam2ub}, derived using a concentration inequality.  

\begin{prop}
\label{thm:ERalgConn}
Let $\epsilon > 0$ and assume $n$ to be even. With probability at least $1-\epsilon$, the algebraic connectivity, $\lambda_{2}$, of an Erd\"os-R\'enyi graph $G(n,p)$ satisfies
\begin{align}
\label{eq:ERlam2ub}
\lambda_{2} \leq  np + 4 n^{-2}  \sqrt{ 2 \log(1/\epsilon) }.
\end{align}
\end{prop}
\begin{proof}
Choose any subset $U\subset V$ with $|U|=\frac{n}{2}$. Equation \eqref{eq:edgeCut} implies that $\lambda_{2} \leq \frac{4 C}{n}$ where $C\sim \mathcal B(\frac{n^{2}}{4},p)$. For $a>0$, we compute
\begin{align*}
\pr \left( \lambda_{2} \geq n p  + a \right) &\leq  \pr \left( 4C/n  \geq n p  + a \right) 
= \pr \left( C - p n^{2}/4    \geq  + an/4 \right) 
\leq \exp \left( - a^{2} n^{4}/32 \right)
\end{align*}
where the last inequality is due to Hoeffding. Setting $a= 4 n^{-2} \sqrt{2 \log(1/\epsilon) }$, we find that
$ \pr \left( \lambda_{2} \geq np + a \right) \leq \epsilon $ as desired. 
\end{proof}
For a random graph $G(n,p)$, the number of edges $m \sim \mathcal B(N,p)$ where $N:=n(n-1)/2$. Thus, $\E[m] = pN$ and we may restate \eqref{eq:ERlam2ub} as: with probability at least $1-\epsilon$, 
\begin{equation}
\label{eq:ERlam2up2}
\lambda_{2} \leq \frac{2 E[m]}{n-1} + 4 n^{-2} \sqrt{ 2 \log(1/\epsilon) }. 
\end{equation}
Indeed,  the first term on the right hand side of \eqref{eq:ERlam2up2} matches the right hand side of  \eqref{eq:lam2ub}.

In Figure \ref{fig:maxAlgConn}, we plot, for $n = 50$ (left) and $n=100$ (right) and $p=.4$ (blue), $p=.6$ (red), and $p=.8$ (green) the value of $m$ vs. $\lambda_{2}$ for 5,000 randomly generated Erd\"os-R\'enyi graphs. The mean values obtained are indicated by circles. 
We use the greedy algorithm  described in \S \ref{sec:compGraphs}  (see Algorithm \ref{alg:Greedy}) with initial graph  taken to be the path with $n$ vertices, $P_{n}$, to compute nearly-optimal graphs with $n$-nodes and $m$-edges. The solid black line in Figure \ref{fig:maxAlgConn} represents the value of 
$\lambda_{2}$ for these graphs. 
Finally, the dashed blue line in Figure \ref{fig:maxAlgConn} represents the upper bound on $\lambda_{2}$ given in \eqref{eq:lam2ub} (compare also to \eqref{eq:ERlam2up2}). 

We observe in Figure \ref{fig:maxAlgConn} that nearly-optimal graphs  have values which are indeed close to the upper bound on the algebraic connectivity, indicating (i) the upper bound is nearly-tight and (ii) the greedy heuristic (Algorithm \ref{alg:Greedy}) produces graphs which are nearly-optimal. We also observe that the algebraic connectivity of nearly-optimal graphs is significantly better than the values for an average Erd\"os-R\'enyi random graph.

\begin{figure}[t!]
\begin{center}
\includegraphics[width=.49\textwidth]{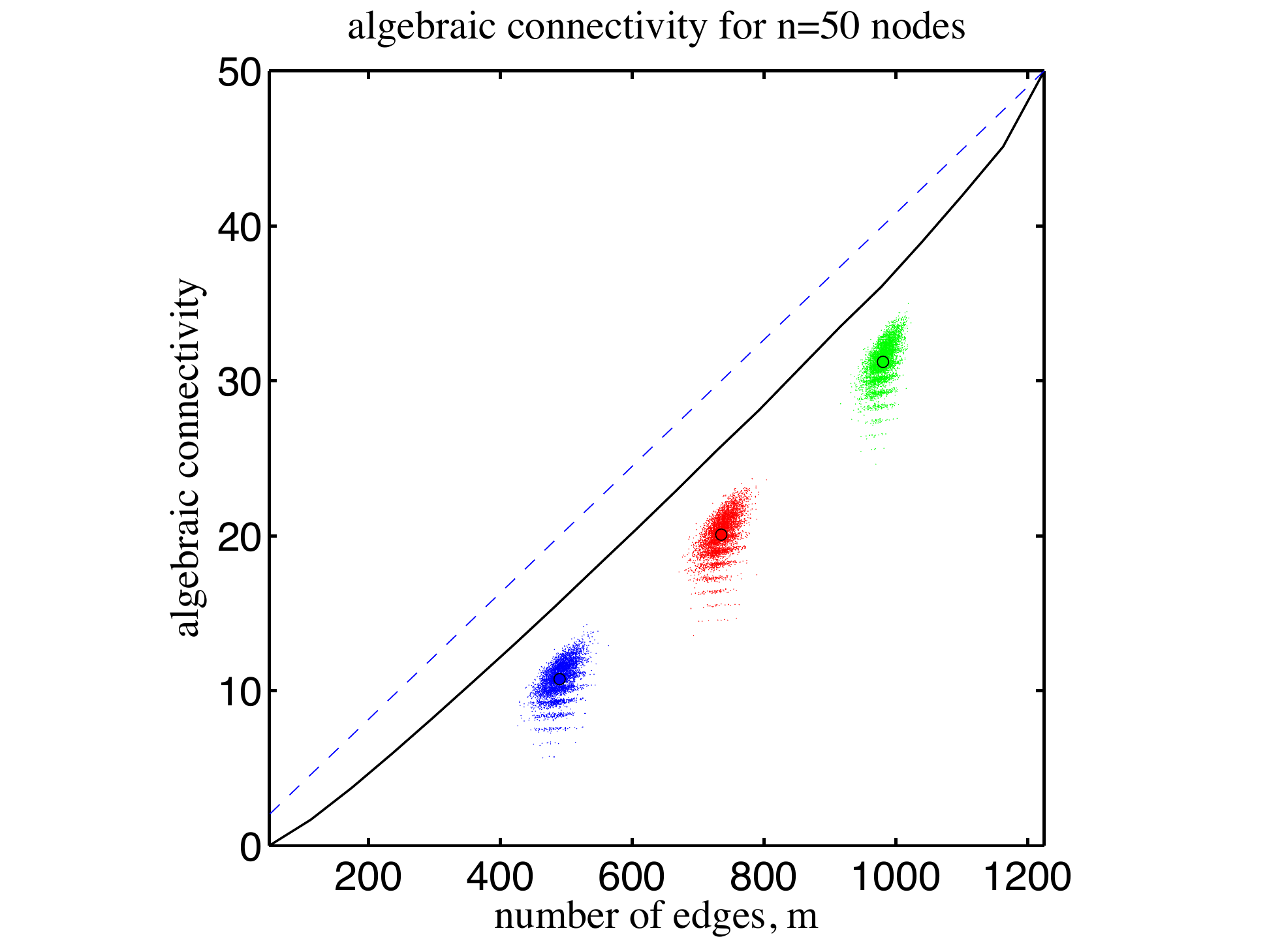} 
\includegraphics[width=.49\textwidth]{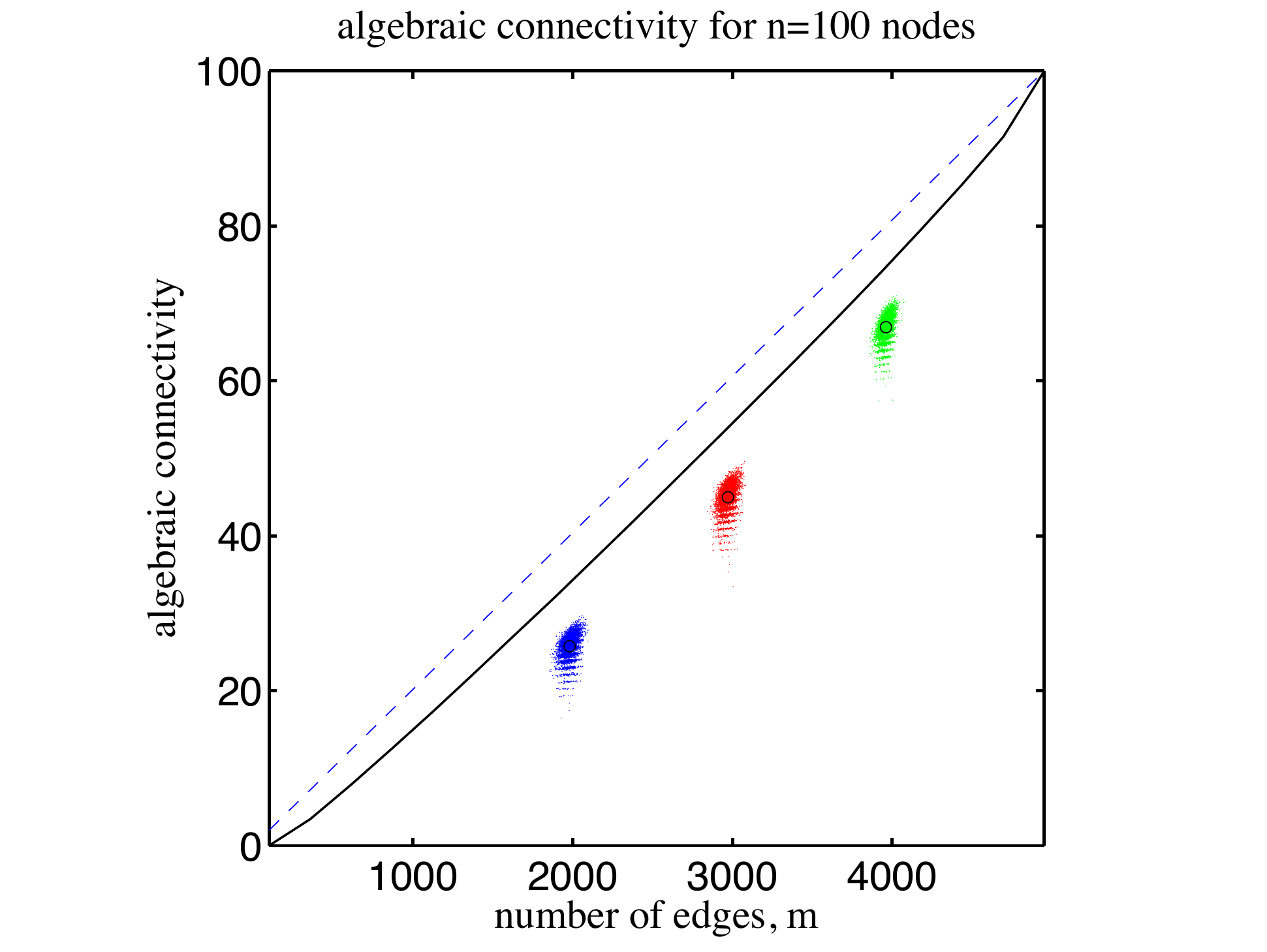} 
\caption{Algebraic connectivity, $\lambda_{2}$ as a function of $m$ for $50$- and $100$-node graphs. 
The dashed blue line  represents the upper bound on $\lambda_{2}$ given in \eqref{eq:lam2ub}. The solid black line represents the nearly-optimal value of $\lambda_{2}$. Finally, for $p=.4$ (blue), $.6$ (red), and $.8$ (green) we give a scatter plot of $(m,\lambda_{2})$ for 5,000 randomly generated Erd\"os-R\'enyi graphs. The mean values obtained are indicated by circles. See \S \ref{sec:compAlgConn}. \label{fig:maxAlgConn}}
\end{center}
\end{figure}

\subsection{Informativeness of the ranking for the  Yahoo! Movie user ratings dataset}
\label{sec:movies}
In this section, we apply the methodology formulated in \S \ref{sec:OSD}, to study the Fisher informativeness of the Yahoo! Movie user rating dataset. We show that the addition of targeted edges can significantly improve the informativeness of the movie rating system.

\paragraph{The dataset} The Yahoo!  Movie user rating dataset consists of a $7,642 \times 11,915$ user-movie matrix where each of the $211,197$ nonzero entries (0.23\% sparsity density) is a 1 to 13 rating \citep{yahooData}.\footnote{34 entries reviewing Yahoo! movie\_id 0 were discarded due to absence in movie content description file.
} Each movie was rated by between 1 and 4,238 users (the average number of reviews per movie is 17.7).
Each user rated between 10 and 1,632 movies (the average number of reviews made by each reviewer is 27.6).
Of the 70,977,655 (movie)  pairs $(i,j)$ where $i>j$, there are  5,742,557 for which a user has given a rating to both movies $i$ and $j$ implying that the pairwise comparisons for the raw dataset are 8.1\% complete. The majority of  movies in the dataset  received relatively few reviews, as reported in Table \ref{tab:MovieFreq}.
\begin{table*}[t!]
{\small
\begin{center}
\begin{tabular}{l | cccccccccc}
\# times movie reviewed  & 1 & 2 & 3 & 4 & 5& 6& 7 & 8 & 9 & $\geq 10$  \\
\hline
occurrences &  4,901 & 1,882 & 897 & 548 & 398 & 316 &237 & 202 & 167 & 2,367
\end{tabular}
\end{center}
\caption{Frequency of reviews for items in the Yahoo!  Movie user rating dataset. See \S \ref{sec:compAlgConn}. }
\label{tab:MovieFreq}}
\end{table*}
The movies which received less than 10 rankings were discarded from the dataset, leaving 2,367  movies, each of which were reviewed by an average of 79.8 users. We then removed 11 users who did not review any of the remaining movies. The remaining 7,631 reviewers reviewed between 1 and 1,220 movies (on average they reviewed 24.8 movies).

\paragraph{Construction of pairwise comparison data from movie-user rating data}
Let $\Sigma$ be the set of Yahoo! users, $V$ be the set of all Yahoo! movies and $r_{i}^{\sigma}$ be the rating given to movie $i\in V$ by user $\sigma\in \Sigma$. For each unordered movie pair $\{i,j\}  \in V^{2}$, we define
$$
\Sigma_{ij} = \{ \sigma \in \Sigma  \text{ who rated both movies } i \text{ and } j \}.
$$
For each movie pair $ \{i,j\} \in V^{2}$, we define $w_{ij}$ to be the number of users who have viewed both movies $i$ and $j$, \ie,  $w_{ij} = | \Sigma_{ij}|$, and $y_{k}$ to be the average difference in movie reviews, written
\begin{equation}
\label{eq:constructPairwise}
y_{ij} = \frac{1}{| \Sigma_{ij} |} \sum_{\sigma \in \Sigma_{ij}}
(r^{\sigma}_{j} - r^{\sigma}_{i}), \,\,
  \text{where } \{i,j\} \in V^{2} \text{ and } i<j.
\end{equation}
Note that the expression in parenthesis is anti-symmetric in the indices $i$ and $j$ and lies in the interval $[-12,12]$. The choice $i<j$ corresponds to the choice in arc direction in \eqref{eq:incidence}. For the Yahoo! Movie user rating dataset, we have 
$n :=|V| = 2,367$, 
$N := \binom{n}{2} = 2,800,161$, 
$m := \|w\|_{0} =  1,884,504$, and 
$M := \|w\|_{1} = 8,322,538$.
Thus, there exists at least one comparison for $m/N=67\%$ of the movie pairs. The mean $w$-weighted degree of each node is given by $2 \cdot M/n =  3,516$. 
A log-histogram of the $w$-weighted degree distribution of the graph representing the pairwise comparison data is given in Fig. \ref{fig:pairCompHist} (top left). 

\begin{figure}[t!]
\begin{center}
  \begin{minipage}[c]{1\textwidth}
  \begin{center}
\includegraphics[width=.4\textwidth,trim=10mm 0mm 8mm 0mm]{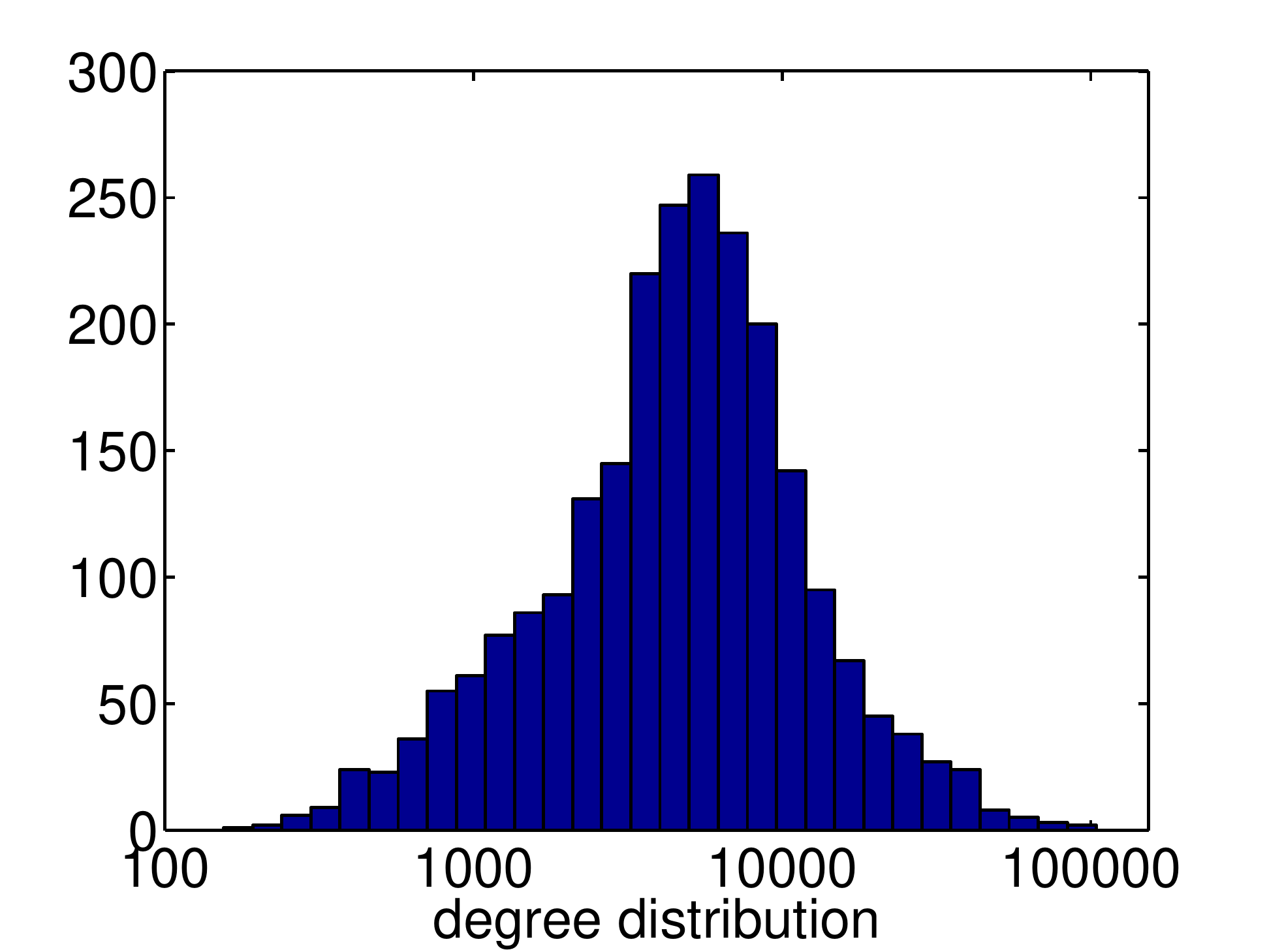} 
\includegraphics[width=.4\textwidth,trim=10mm 0mm 8mm 0mm]{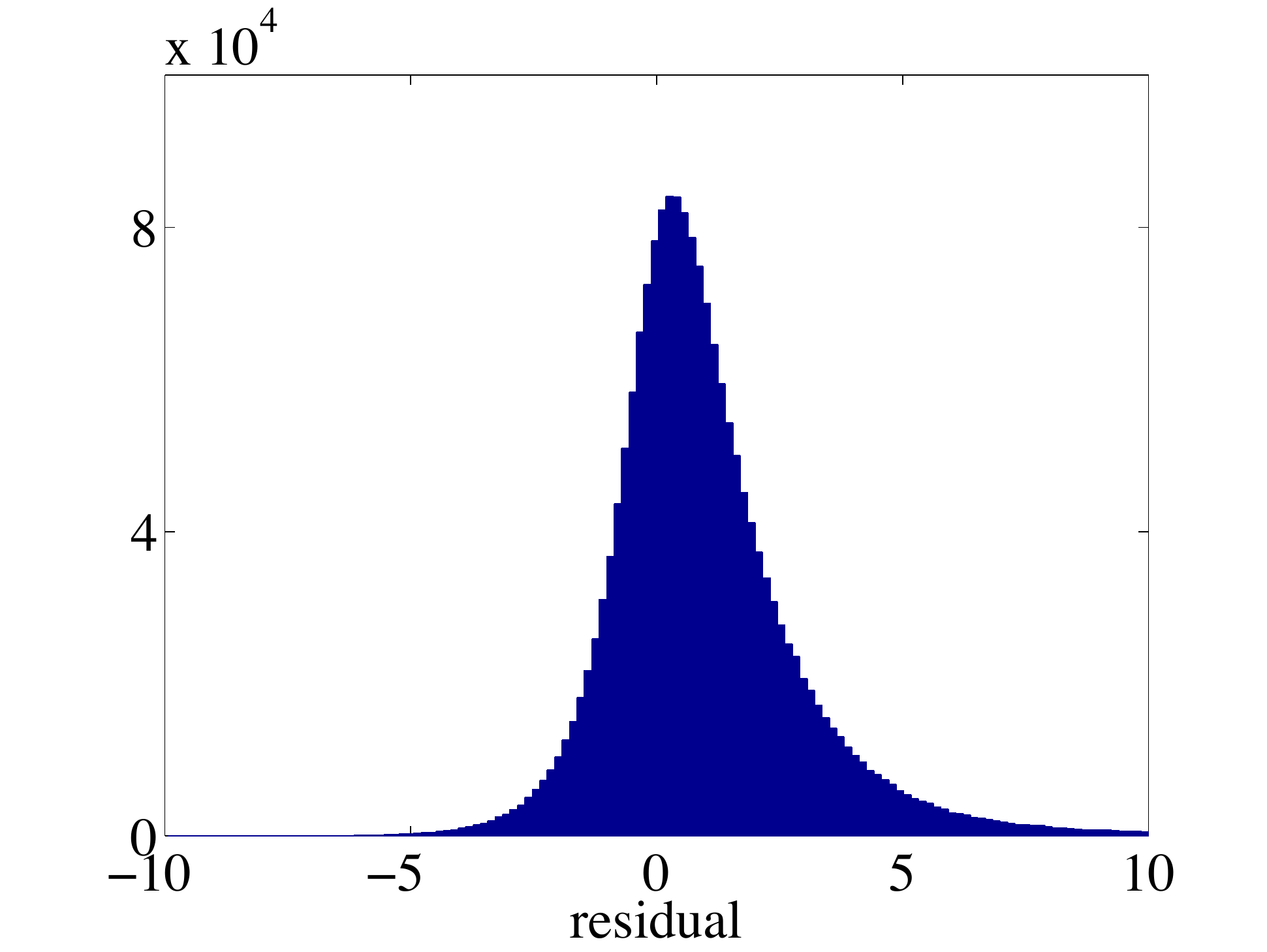} 
\end{center}
\end{minipage}
\medskip

\begin{minipage}[c]{0.45\textwidth}
\begin{center}
{\small
\begin{tabular}{ l  |  l  }
$\hat \phi_{w}$ & Movie Name \\
\hline
 4.46 & It's a Wonderful Life (1946) \\
    4.45 &   Singin' in the Rain (1952) \\
    4.34 &  Rear Window (1954) \\
    4.11 &     24: Season 1 (2002) \\
    3.96 &  The Longest Day (1962) \\
    3.94 &    The Man Who Shot Liberty Valance (1962) \\
    3.92 &     Rebecca (1940) \\
    3.87 &     Friends - The Complete Fourth Season (1997) \\
    3.79 &     Lady and the Tramp (1955) \\
    3.79 &  It Happened One Night (1934) 
\end{tabular}}
\end{center}
\end{minipage}

\end{center}
\caption{{\bf (top left)} A log-histogram of the $w$-weighted degree distribution for the graph representing the Yahoo! movie pairwise comparison data. {\bf (top right)} A histogram of the residual, $y-B\hat{\phi}_w$, where $\hat{\phi}_w$ is the least squares ranking. {\bf (bottom)} Top 10 movies and ranking, $\hat\phi_{w}$. See \S \ref{sec:movies}. 
\label{fig:pairCompHist}}
\end{figure}

\begin{figure}[t!]
\begin{center}
  \begin{minipage}[c]{0.38\textwidth}
  \begin{center}
\includegraphics[width=1\textwidth]{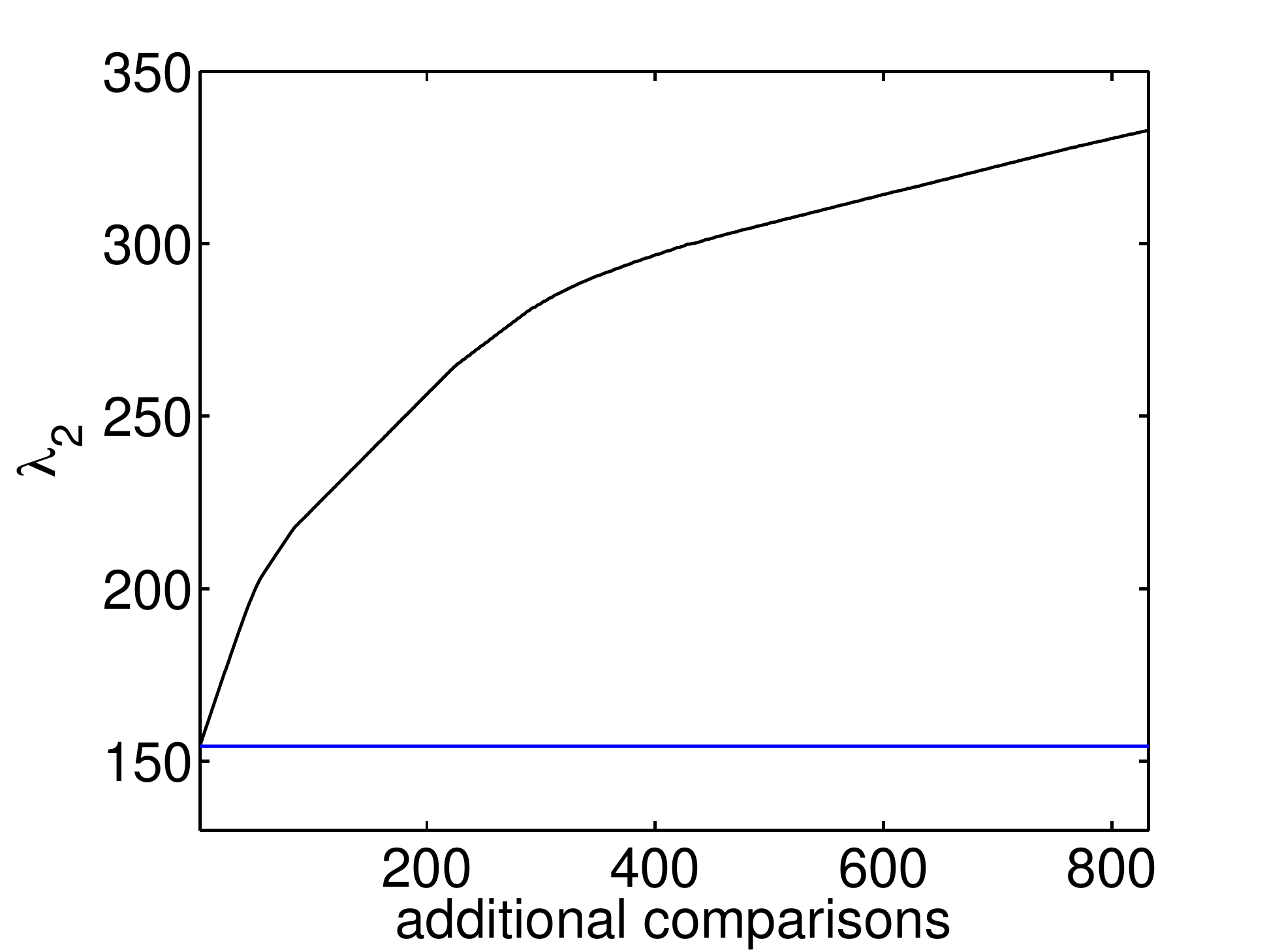} 
\end{center}
\end{minipage}
\begin{minipage}[c]{0.45\textwidth}
\begin{center}
\vspace{.5cm}
{\small
\begin{tabular}{c | c c c }
&\multicolumn{3}{| c}{additional comparisons} \\
method of   &  original& $.005\% \ M$  &  $.01\% \ M$\\
data collection & $\xi = 0$ & $\xi = 416$ & $\xi = 832$   \\
\hline
random  & 154.38 & 154.38 & 154.38 \\
optimal & 154.38 & {\bf 298.44}&{\bf  332.78} \\
\hline
upper bound \eqref{eq:lam2ub} & 7,035 & 7,035 & 7,036
\end{tabular}}
\end{center}
\end{minipage}
\end{center}
\caption{{\bf (left)} The informativeness of the ranking, $\lambda_{2}(w)$, as a small number ($.01\% \cdot M$) of targeted pairwise comparisons (black) and  randomly selected pairwise comparisons (blue) are added. 
{\bf (right)} The value of $\lambda_{2}(w)$ for this augmented dataset and the upper bound on $\lambda_{2}$ given in \eqref{eq:lam2ub}. The change in informativeness for randomly added data is unappreciable compared to a 2.2 fold increase for targeted data. 
See \S \ref{sec:movies}. 
\label{fig:addComps}}
\end{figure}

\paragraph{The least squares ranking}
A ranking is obtained by solving the least squares problem, \eqref{eq:R}, using Matlab's \texttt{lsqr} function. The top ten movies found are given in Figure \ref{fig:pairCompHist}.  The relative residual norm of the least squares estimator, $\hat \phi_{w}$, is
$
\frac{\|B \hat \phi_{w} - y \|_{w}} { \| y \|_{w} } =  0.53.
$
In Fig. \ref{fig:pairCompHist} (top right), we plot a histogram of the residual,  $y-B\hat{\phi}_w$. For this pairwise comparison dataset, the normality assumption in Prop. \ref{prop:minLam2} is reasonable. 

The informativeness of the ranking is $ \lambda_{2}(w) = [\text{Var}(\hat \phi_{w})]^{-1} = 154.38$. This value is  small compared to the upper bound given in \eqref{eq:lam2ub}, $\lambda_{2}(w) \leq  \frac{2 M}{n-1} = 7,036$. 
We next demonstrate that the Fisher  information can be significantly improved by the addition of a small number of targeted pairwise comparisons.

\paragraph{Targeted data collection} We apply the optimal experimental design approach developed  in \S \ref{sec:OSD} to improve the Fisher information of the least squares ranking.  To approximate the solution of \eqref{eq:minLam2}, we use the greedy algorithm described in Algorithm \ref{alg:Greedy}. The second eigenpair of the graph Laplacian  is computed using Matlab's \texttt{eigs} function, initialized using the eigenvector from the previous iteration.   We choose a very modest value of pairwise comparison edges to add, $\xi = .01\% \cdot M = 832$ edges. The results are given in Fig. \ref{fig:addComps}. The addition of the targeted pairwise comparisons leads to an increase in the second eigenvalue of the $w$-weighted graph Laplacian by a factor of $2.2$. The maximum increase for the addition of a single pairwise comparison is $\approx 1$, less than the upper bound given in  \eqref{eq:maxIncrease}. We observe in Fig. \ref{fig:addComps}, that the rate of information increase slows as more pairwise comparisons are added. For a comparison, we also consider the addition of randomly chosen movie pairs. For this modest value of additional edges, $\xi$, the effect of the informativeness of the ranking is unappreciable.

\begin{figure*}[tf!]
\vspace{1.5cm}
\centering
\begin{minipage}[t]{.75\textwidth}
   \begin{center}
   		\includegraphics[width=.96\textwidth]{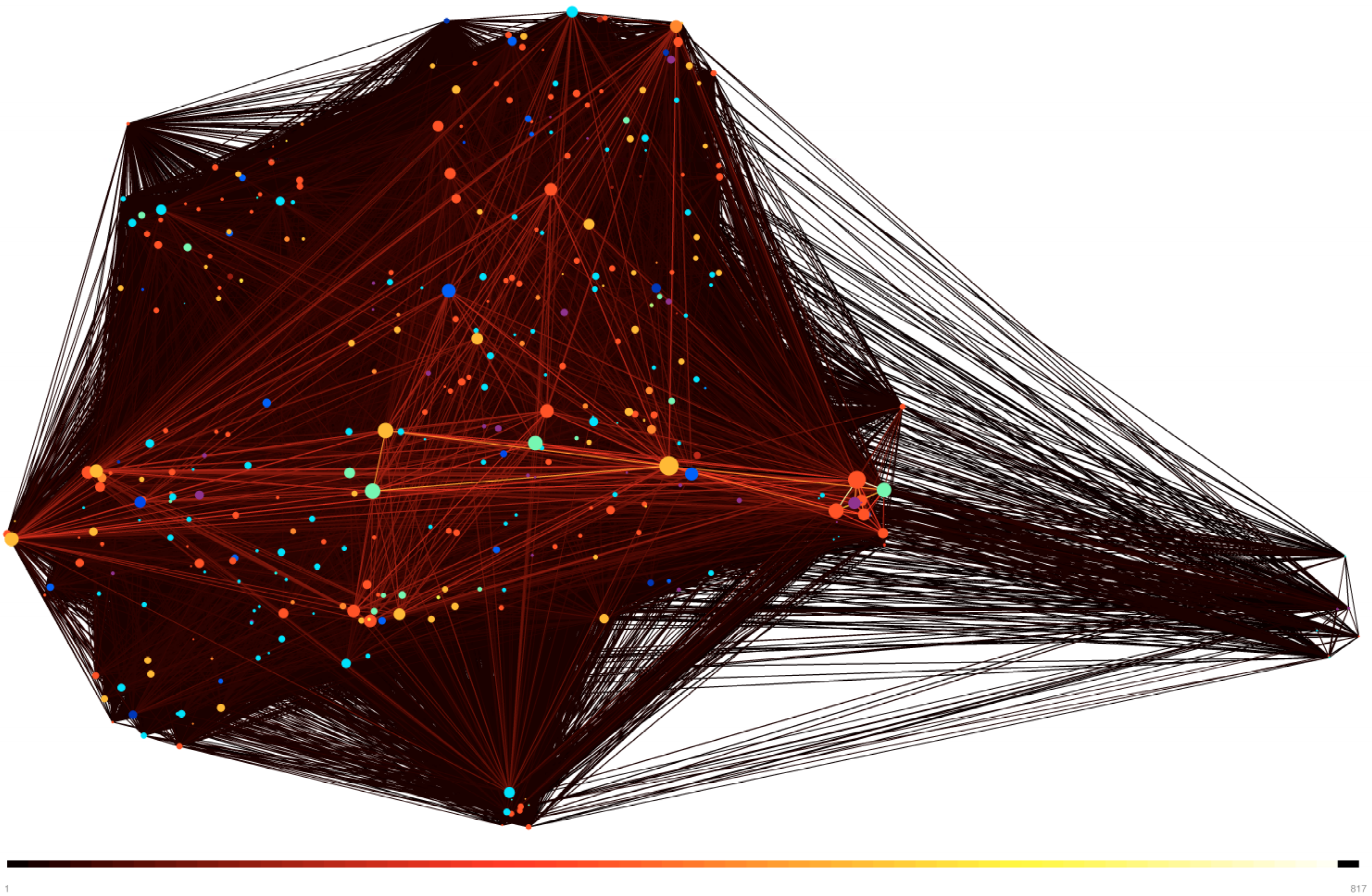} \\
		\bigskip		
   		\includegraphics[width=.96\textwidth]{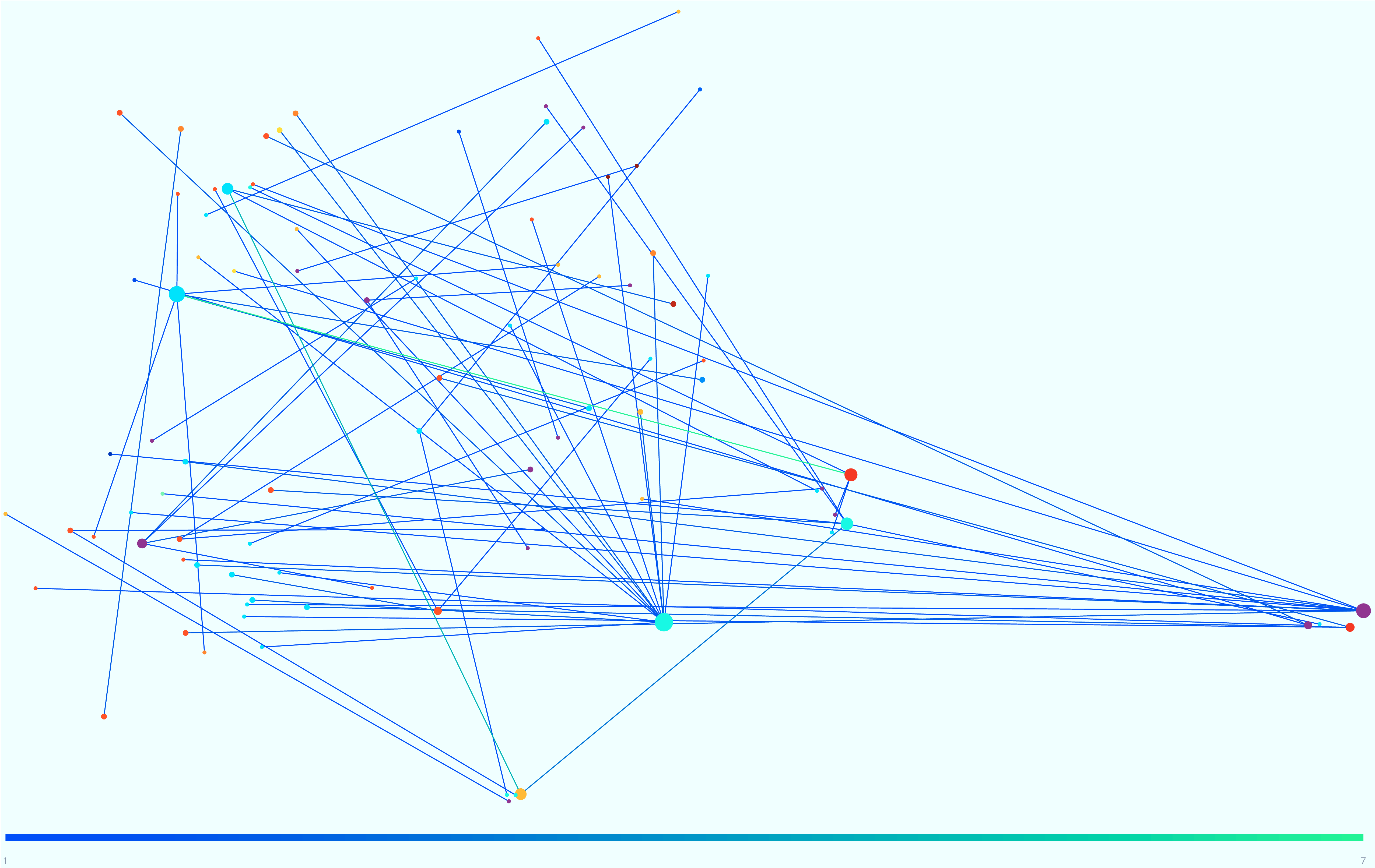} \\
   \end{center}
\end{minipage}%
\begin{minipage}[t]{.25\textwidth}
   \vspace{-8em}
   \begin{center}
   		\hspace{4em}
    		\includegraphics[width=\textwidth]{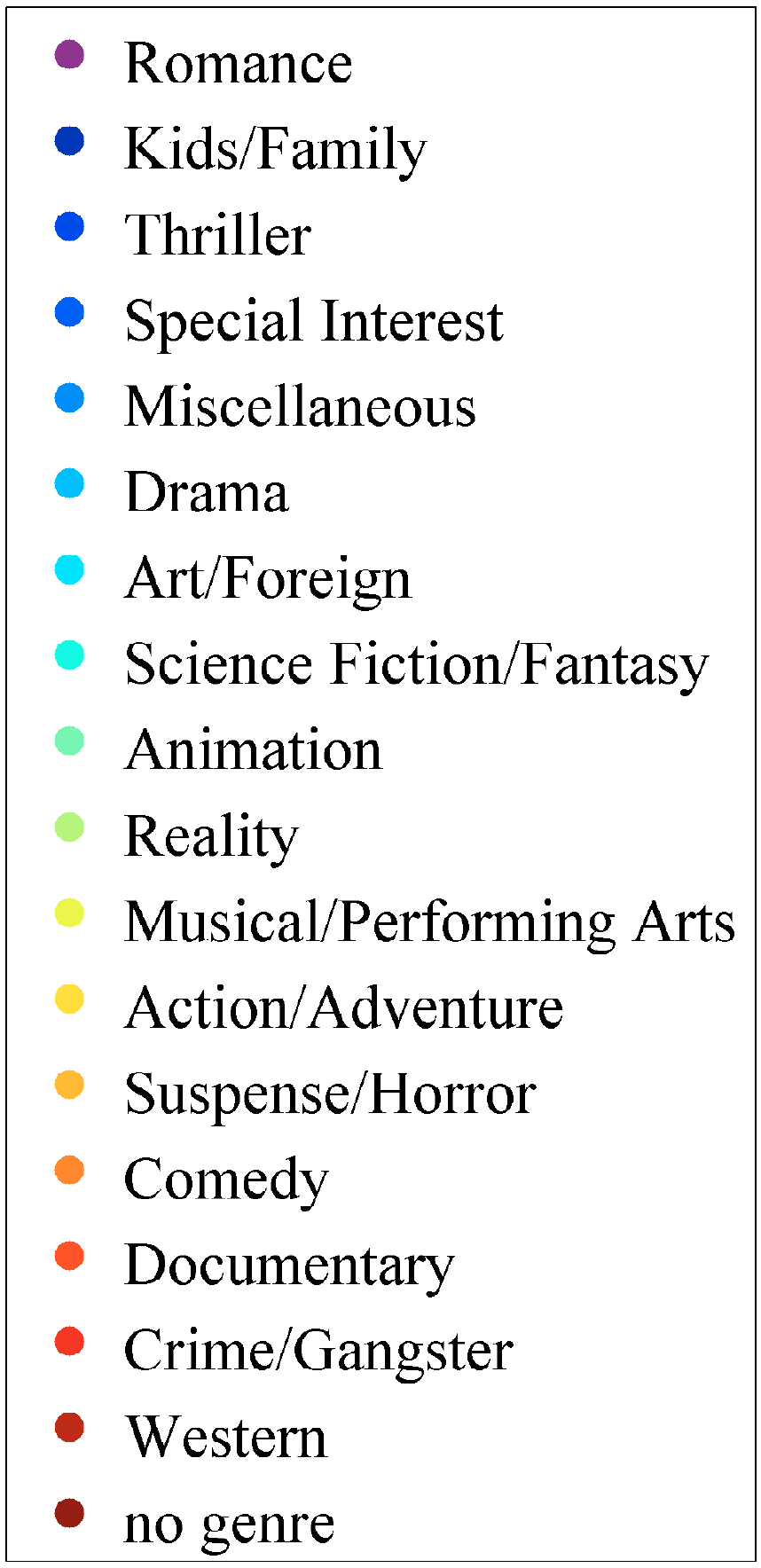}
	\end{center}
\end{minipage}\\%
\caption{{\bf Yahoo! Movie ratings and targeted data collection}. 
{\bf (top)} A (15\% randomly chosen) subset of the pairwise comparison graph for the Yahoo! user-movie database. Nodes represent movies, node size reflects weighted degree (\ie, number of comparisons with other movies), and node color indicates genre (see legend). Edges represent weighted pairwise comparisons colored by edge weights (\ie, number of comparisons). 
{\bf (bottom)} Pairwise comparisons targeted for collection to improve the informativeness of the least squares ranking. Targeted comparisons are colored by weight (multiplicity).  See \S \ref{sec:movies}.
\label{fig:clusterYahoo}}
\end{figure*}

Finally, we use graph visualization via spectral clustering to illustrate the pairwise comparison and targeted data. In Fig. \ref{fig:clusterYahoo}(top) we plot the  pairwise movie comparisons obtained from the Yahoo! user-movie database. 
In Fig. \ref{fig:clusterYahoo}(bottom) we plot the proposed pairwise comparisons, targeted to improve the informativeness of the rating system. 
To enhance the readability of the graph representation, we plot only $15\%$ randomly selected nodes ($356$ of $n=2367$) and the interconnecting edges ($45,327$ of $m=1,884,504$). 
Figure  \ref{fig:clusterYahoo}(top) was generated as follows. 
First normalized spectral clustering (based on $k$-means) was used to detect clusters of movies. 
Next, the Fruchterman-Reingold algorithm was used to generate reasonable positions for the  movie clusters and the Kamada-Kawai algorithm was used to place movies within the clusters \citep{Traud:2009}. 
The node placement was obtained using the full dataset. Finally, the weighted graphs were plotted using \texttt{wgPlot} \citep{wgPlot}. Figure  \ref{fig:clusterYahoo}(bottom) was then generated using the same node placements as in Figure  \ref{fig:clusterYahoo}(top).  

A comparison of the top and bottom panels of Fig. \ref{fig:clusterYahoo} shows that the primary improvement to informativeness arises from the addition of edges which connects two relatively weakly connected components of the graph. With 4 exceptions, each targeted movie pair is only incremented once; it isn't generally advantageous to add an edge  multiple times.

\subsection{2011-12 NCAA Division I football schedule} \label{sec:football}
Recall from  \S \ref{sec:intro} that in sports the optimal pairwise data collection problem in equivalent to designing the schedule. 
In this section, we study the 2011-12 NCAA Division 1 football schedule, downloaded from Massey Ratings.\footnote{\url{http://masseyratings.com/scores.php?t=11590&s=107811&all=1&mode=2&format=0}} 
The NCAA Division 1 Football League is divided into the Football Bowl Subdivision (FBS) and Football Championship Subdivision (FCS).\footnote{These were formally known as Division 1-A and 1-AA respectively.} 
The FBS is further decomposed into 12 conferences and the FCS into 15. 
Of the 246 teams in Division 1, 120 belong to FBS and 126 belong to FCS. 
Lafayette College is a member of FBS, however every opponent of Lafayette during the 2011-12 season was a member of the FCS. For our purposes, it is more convenient to reclassify Lafayette as a member of FCS and thus, in what follows, FBS has 119 teams and FCS has 127. 
There were $m=1430$ games among the Division 1 teams and $m=693$ games among the FBS teams. 

For static schedules, an important  statistic is the the ratio of the total number of games played to the total number of teams. For example, in Major League Baseball (MLB), there are 30 teams, divided into two leagues: the American League (14 teams) and the National League (16 teams). During the regular season, each team plays approximately 160 games, primarily against teams within the same division. Thus, within each league, teams play an average of $160/15 \approx 10$ times. With so many games and equal strength of schedule among teams, it is intuitive that the scheduling has little effect on the rankings. And, in fact, MLB simply uses win/loss percentages for ranking purposes. In the NCAA football considered here however, there are 120 teams in the NCAA Football Bowl Subdivision (FBS) and each team plays approximately 6 games per year within FBS. Thus each team only plays roughly $5\%$ of the other teams. There are several rankings for NCAA football which are generated either mathematically or by expert opinion and then aggregated to determine official rankings and select teams to compete in the prestigious end-of-season ``bowl games''.  The fact that these rankings generally disagree and that none of them is more reliable than the others suggest that none of them are very informative. It is this situation, where there  are relatively few games compared to the number of teams, that the schedule has a large effect on the rankings.

\begin{figure}[hf!]
\begin{center}
\includegraphics[width=.92\textwidth]{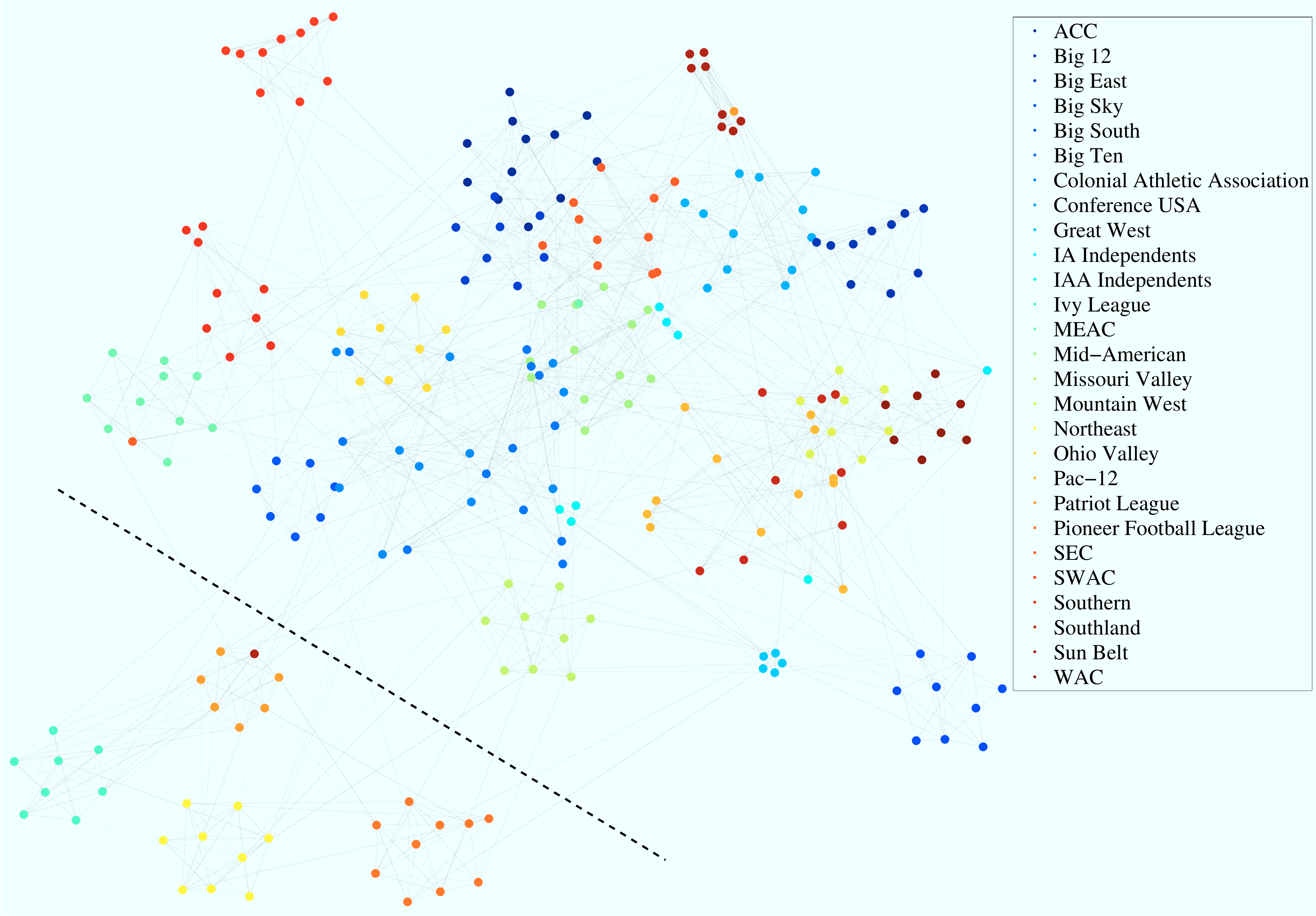}\\
\includegraphics[width=.52\textwidth]{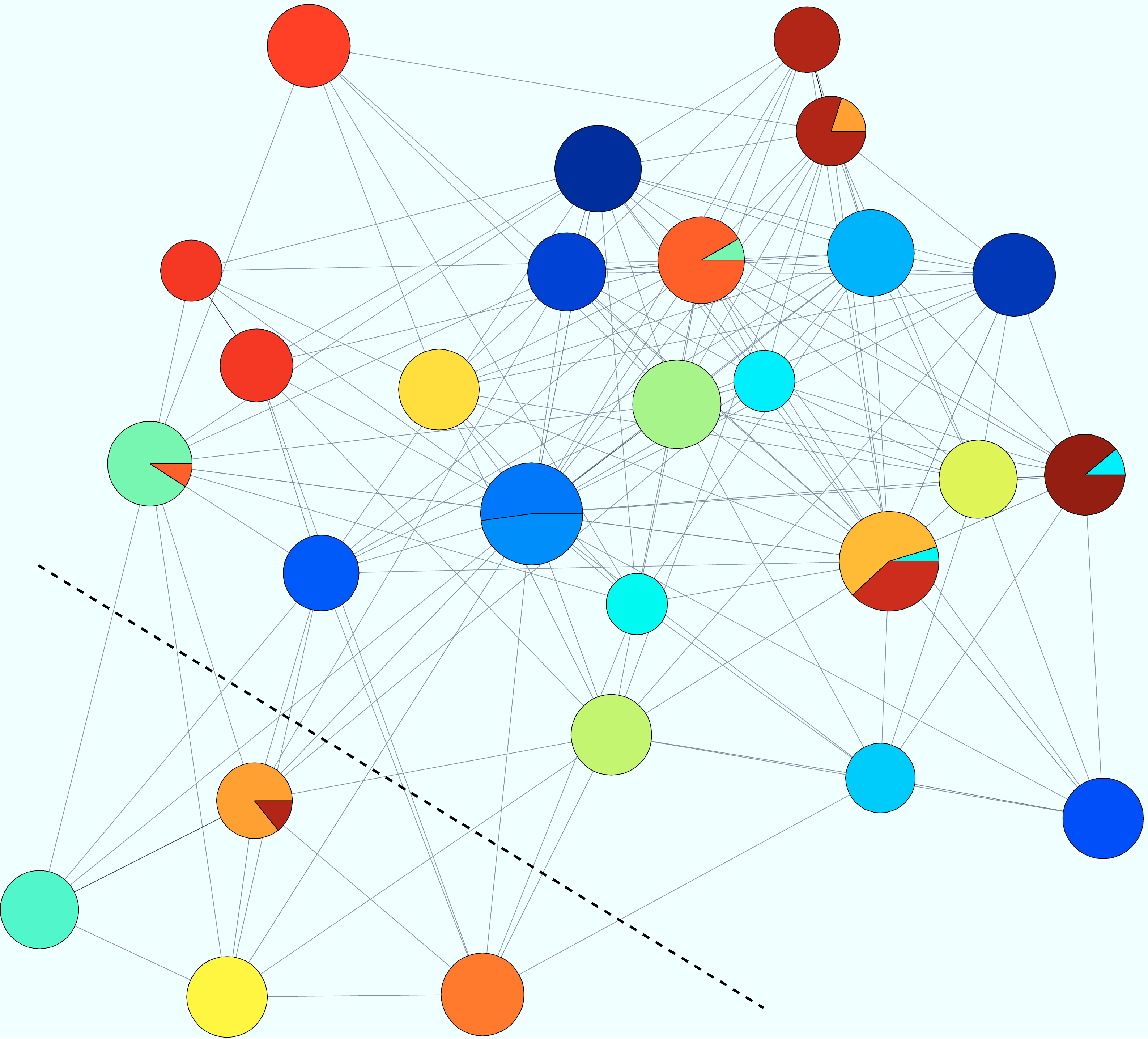}
\caption{\textbf{2011-12 NCAA Division 1 (FBS and FCS) football schedule}. Graph representation of schedule via spectral clustering by games, \textit{top:} vertices represent teams, edges represent games, coloring indicates conference membership. \textit{bottom:} community detection of teams (represented using pie-graphs) reveals  that teams primarily play within their own conference. The dashed lines indicate an edge cut which is discussed in the text. See \S \ref{sec:football}.
\label{fig:cluster}}
\end{center}
\end{figure}

\paragraph{Data visualization via spectral clustering} \label{sec:FBdatavisualization}
We use the  data visualization method described below to demonstrate that NCAA Division 1 teams primarily play against other teams within their own conference. We then  show that this clustering of  teams by conference results in the graph having poor algebraic connectivity.

We first use normalized spectral clustering to detect communities within the teams \citep{Shi:2000}. This, in turn, relies on the $k$-means algorithm where $k$ is the desired number of communities (27 for Division 1 and 12 for Division 1 FBS). Then, using the Matlab toolbox described in \citep{Traud:2009}, the Fruchterman-Reingold algorithm finds an optimal placement of the communities and the Kamada-Kawai algorithm is used for the placement of nodes within each community. 
The mean within-cluster sum of point-to-centroid distances for the $k$-means clustering obtained for the Divsion 1 and Division 1 FBS data is $0.147$ and $0.133$  respectively. 

In Figures \ref{fig:cluster} and \ref{fig:cluster1A}, we plot the 2011-12 NCAA Division 1 and Division 1 FBS football schedules respectively. In \ref{fig:cluster}(top) and \ref{fig:cluster1A}(top), the vertices represent teams, the edges represent games, and each vertex (team) is colored by conference membership. In \ref{fig:cluster}(bottom) and \ref{fig:cluster1A}(bottom), the vertices represent the spectrally clustered communities and the edges represent the community interactions. 
 We observe from Figures \ref{fig:cluster} and \ref{fig:cluster1A} that the teams primarily play within their own conference, which has implications discussed  below.

We next compare the value of the algebraic connectivity for these schedules with schedules from Erd\"os-R\'enyi random graphs and  proposed nearly-optimal schedules.

\begin{figure}[hf!]
\begin{center}
\includegraphics[width=.92\textwidth]{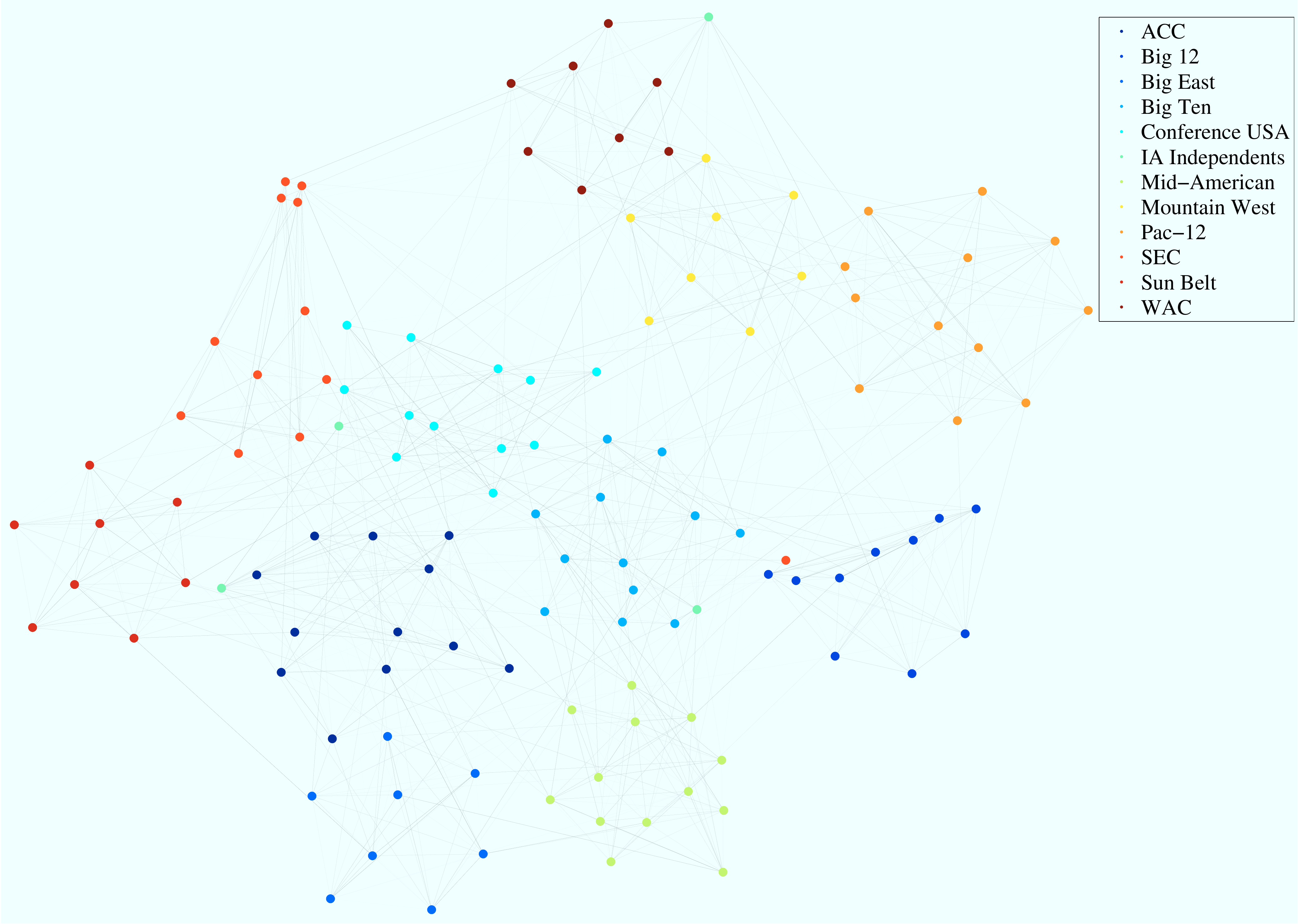}\\
\includegraphics[width=.52\textwidth]{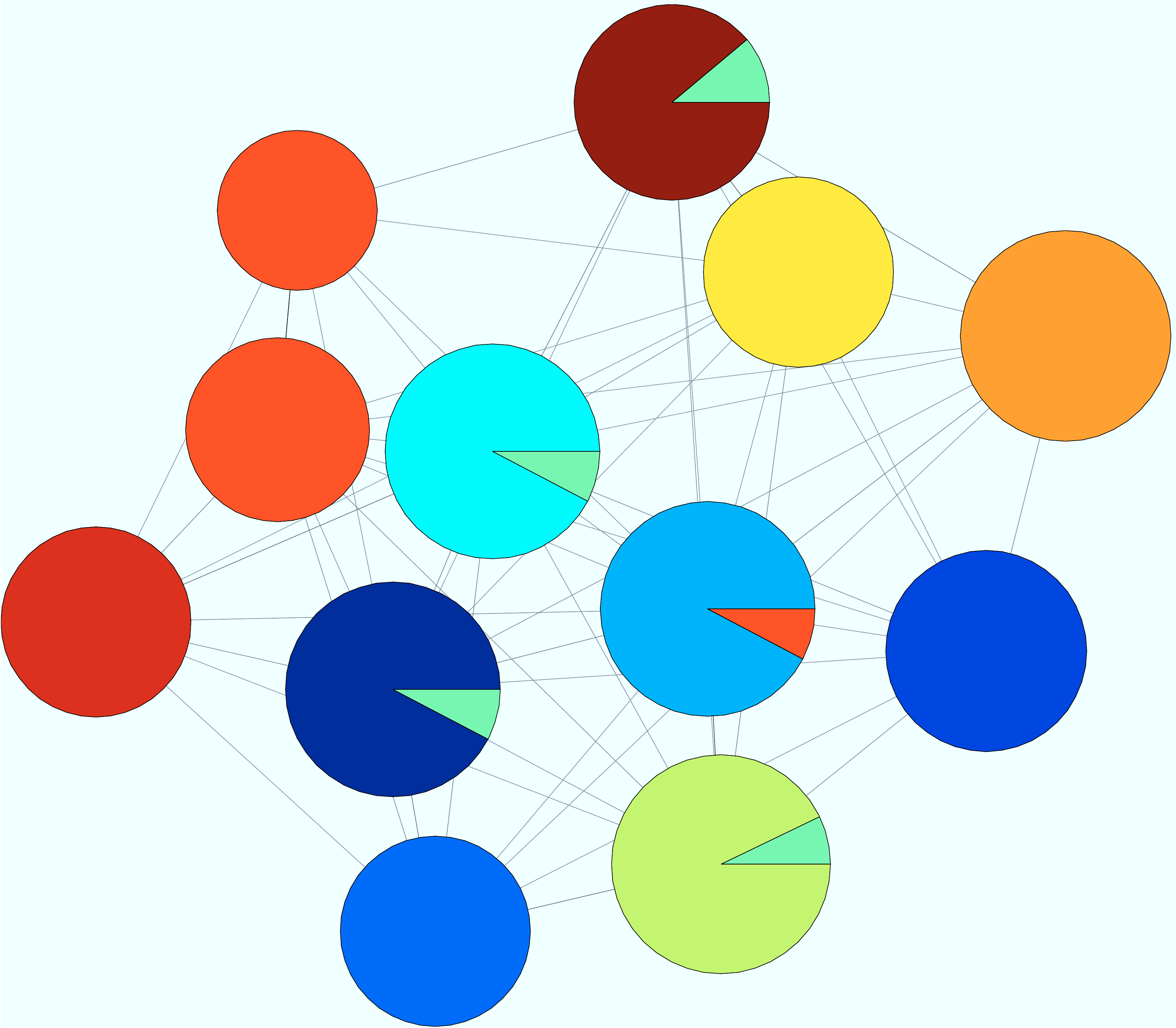}
\caption{\textbf{2011-12 NCAA Division 1 (only FBS) football schedule}. Graph representation of schedule via spectral clustering by games, \textit{top:} vertices represent teams, edges represent games, coloring indicates conference membership. \textit{bottom:} community detection of teams (represented using pie-graphs) reveals that teams primarily play within their own conference. 
See \S \ref{sec:football}.
\label{fig:cluster1A}}
\end{center}
\end{figure}

\paragraph{Comparison of NCAA Division 1, Erd\"os-R\'enyi random,  and nearly-optimal schedules} In the introduction, we noted that there are several common scalar measures of $\text{Var}(\hat{\phi}_{w})$, three of which are given in \eqref{eq:otherOptCond}. In this section, we compare these various measures for the NCAA Division 1, Erd\"os-R\'enyi random,  and nearly-optimal schedules. 

More concretely, let $w$ be a given schedule (defining a graph on $n$ vertices) and define the graph Laplacian: $\Delta_{w} := B^{t} [\text{diag}(w)] B$. Define the following three functions of $w$:
\begin{subequations}
\label{eq:3ObjFun}
\begin{align}
\label{eq:3ObjFuna}
J_{E}(w) &:= \lambda_{2}(w) \\
\label{eq:3ObjFunb}
J_{A}(w) &:= \left[ \frac{1}{n}\text{tr}(\Delta_{w}^{\dag}) \right]^{-1}= \left[ \frac{1}{n} \sum_{i\geq2} \frac{1}{\lambda_{i}(w)} \right]^{-1} \\
\label{eq:3ObjFunc}
J_{D}(w) &:= \log [\text{det}(\Delta_{w}) ]^{\frac{1}{n}} = \frac{1}{n} \sum_{i\geq 2} \log[ \lambda_{i}(w)] 
\end{align}
\end{subequations}
To obtain quantities more comparable to those for the E-optimality condition, for $J_A(w)$ we have used the harmonic mean of the eigenvalues rather than the negative of the inverses as in \eqref{eq:FIrankb} and for $J_D(w)$, we have taken the logarithm of the determinant in \eqref{eq:FIrankc}.  An interpretation of the three quantities defined in \eqref{eq:3ObjFun} in terms of the graph is given in Remark \ref{rem:OtherOptCond}. 

For the Division 1 and Division 1 FBS schedules, we compute the various measures of the quality of schedule given in \eqref{eq:3ObjFun} and record them in Table \ref{tab:EADvals}. 
We also plot $J_{E}(w)$ given in \eqref{eq:3ObjFuna} in Fig. \ref{fig:maxAlgConn124} by a red diamond. 
We next discuss schedules for which we compare the Division 1 and Division 1 FBS schedules in Table \ref{tab:EADvals} and Fig.  \ref{fig:maxAlgConn124}.

The expected number of edges for a $G(n,p)$ Erd\"os-R\'enyi random graph is $p N$ where $N :=\binom{n}{2}$. To compare to the football schedules, we take $p = m/N$ and consider the family of random graphs, $G(n,m/N)$. 
For $n=119$  and $m = 693$, we choose $p=m/N\approx0.0987$ which is approximately $2.5$ times  the threshold for connectivity,  $p_{c} = \log(n)/n \approx 0.0402$. 
For $n=246$  and $m = 1430$, we choose $p=m/N\approx0.0475$ which is approximately $2.1$ times  the threshold for connectivity,  $p_{c} = \log(n)/n \approx 0.0224$. 
In Table \ref{tab:EADvals}, we tabulate the expected values of the three quantities given in \eqref{eq:3ObjFun} for $G(n,m/N)$ graphs, obtained by averaging over a sample size of 1000. 
Similar to \S \ref{sec:compAlgConn}, in Fig.  \ref{fig:maxAlgConn124}, we give a scatter plot of $(m,\lambda_{2})$ for $G(n,m/N)$ graphs and indicate the mean values with a blue circle. 

As in \S \ref{sec:compAlgConn} and \S \ref{sec:movies}, we again use the greedy algorithm described in \S \ref{sec:compGraphs}  (see Algorithm \ref{alg:Greedy})  to compute graphs with $n$ nodes and $m$ edges which nearly-maximize $J_{E} = \lambda_{2}$. We then evaluate all three quantities given in \eqref{eq:3ObjFun} for these graphs and tabulate these values in 
 Table \ref{tab:EADvals}.
 The solid black line in Fig. \ref{fig:maxAlgConn124} is the best value of $J_{E} = \lambda_{2}$ obtained. Finally, the dashed blue line in Fig. \ref{fig:maxAlgConn124} represents the upper bound on $\lambda_{2}$ given in \eqref{eq:lam2ub}. 

We observe in Fig.  \ref{fig:maxAlgConn124} and Table \ref{tab:EADvals}  that the schedules which nearly-maximize $J_{E}(w) = \lambda_{2}$ have significantly larger values of $J_{E}$ than the NCAA Division 1 and Division 1 FBS schedules. In fact, the NCAA schedules have worse values than schedules associated with Erd\"os-R\'enyi random graphs of the same size. 
Furthermore, we show in Table \ref{tab:EADvals} that schedules which maximize $J_{E}$ also have larger  values of $J_{A}$ and $J_{D}$. That is, the schedules which are good in the sense of E-optimality are also good schedules in the sense of D- and E-optimality as defined in \eqref{eq:otherOptCond}.

The reason for the relatively poor value of $J_{E}(w) = \lambda_{2}$ for the NCAA Division 1 and Division 1 FBS schedules can be understood from Figures \ref{fig:cluster} and \ref{fig:cluster1A}.
Figures \ref{fig:cluster} and \ref{fig:cluster1A} reveal that teams primarily play within their own conference.  This results in a small edge cut between a conference (or set of conferences) and its vertex complement, which, by \eqref{eq:edgeCut}, implies a small algebraic connectivity. For example,  the edge cut indicated by the dashed line in Fig. \ref{fig:cluster} (entire  NCAA Division 1 schedule) results in an upper bound on the algebraic connectivity of 1.297. The edge cut obtained by considering the set consisting of teams in the SWAC conference yields an upper bound equal to 1.043. Both of these bounds are already less than the expected value of $\lambda_{2}$ for Erd\"os-R\'enyi random graphs of comparable size (compare with the top part of the first column in Table \ref{tab:EADvals}). To summarize, the NCAA primarily schedules games among teams within the same conferences and this reduces  the informativeness of the rankings. 

The schedule design methodology advocated in Eq. \eqref{eq:minLam2} is  flexible in the following two senses: (i) The optimal schedules contain symmetry with respect to permutations in the seeding of the teams. This problem has been studied previously for tournaments; see  the discussion in \S \ref{sec:relwork}. (ii) The optimal schedule is \emph{not} time dependent and thus the scheduling of future games does \emph{not} depend on past game performances, \ie, the schedule is completely known before the season begins and the games may be played in \emph{any} order. These properties  can be  exploited in the further  design of the schedule.    

\begin{figure}[t!]
\begin{center}
\includegraphics[width=.47\textwidth]{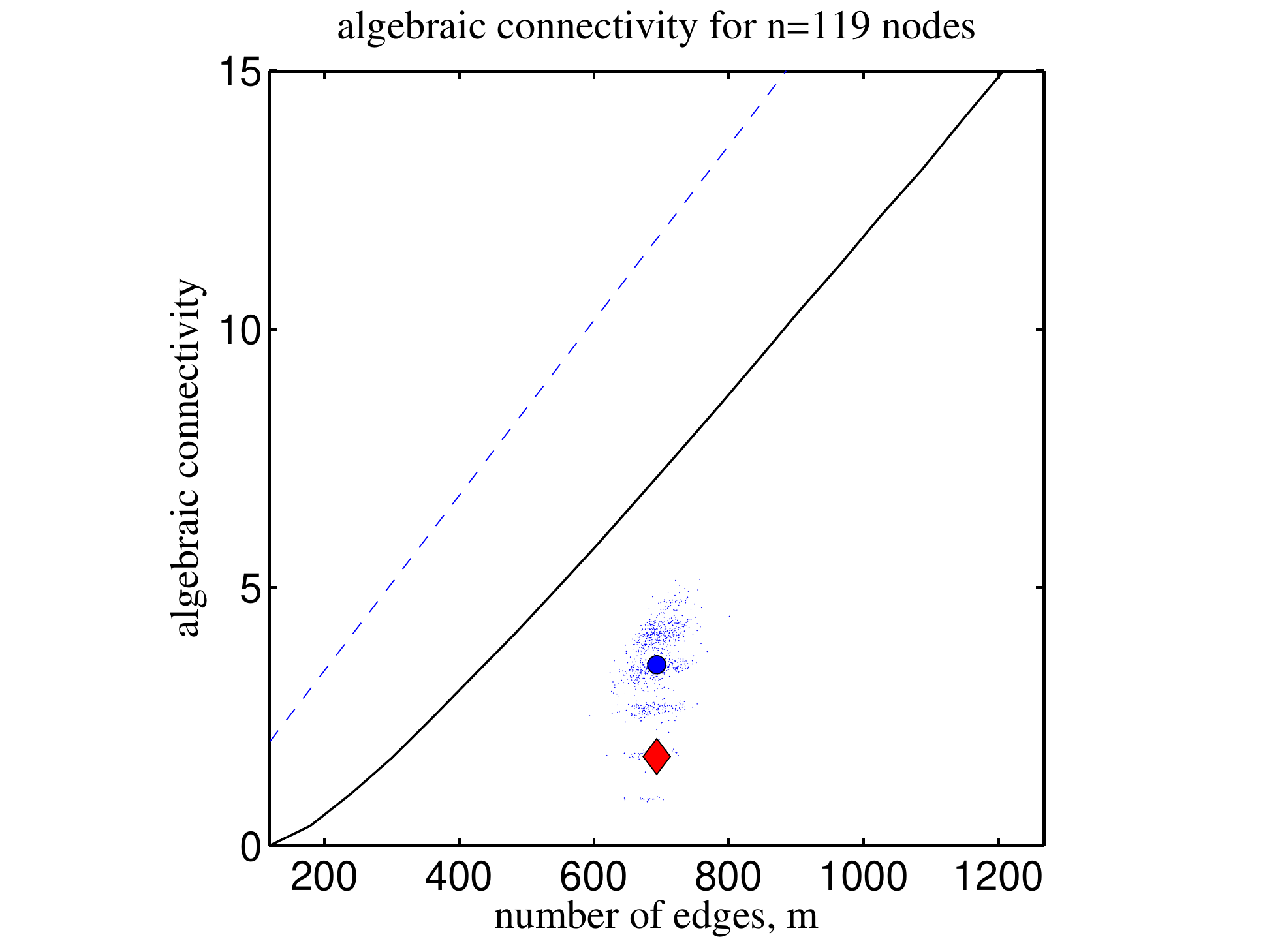}
\includegraphics[width=.47\textwidth]{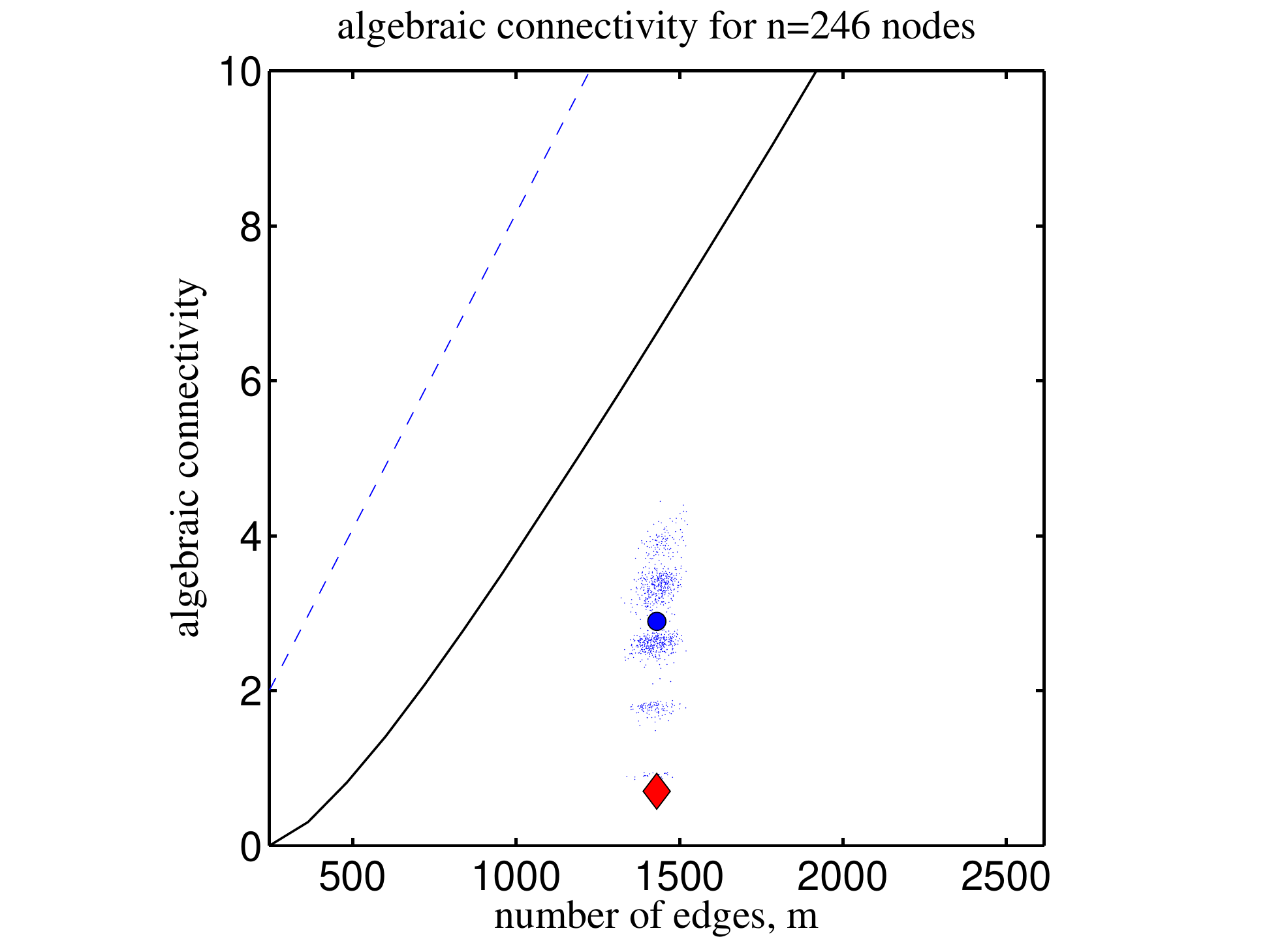}
\caption{A comparison of  $J_{E}(w) = \lambda_{2}$ defined in \eqref{eq:3ObjFuna} for the Division 1 and Division 1 FBS schedules, Erd\"os-R\'enyi random schedules, and schedules which nearly-maximize $ \lambda_{2}$.
The red diamonds represents the 2011 NCAA Division 1 (right) and Division 1 FBS (left) football schedule. 
The solid black lines represent the nearly-optimal values of $\lambda_{2}$ obtained for $n=119$ (left) and $n=246$ (right).
The dashed blue lines  represent the upper bound on $\lambda_{2}$ given in \eqref{eq:lam2ub}.
The blue dots represent a scatter plot of $(m,\lambda_{2})$ for 1,000 randomly generated Erd\"os-R\'enyi graphs, $G(n,m/N)$. The mean values  are indicated by blue circles. 
See \S \ref{sec:football}. 
\label{fig:maxAlgConn124}}
\end{center}
\end{figure}

\begin{table}[hf!]
\begin{center}
\begin{tabular}{l | c c c }
& $J_{E}(w)$ in \eqref{eq:3ObjFuna}  & $J_{A}(w)$  in \eqref{eq:3ObjFunb}  & $J_{D}(w)$  in \eqref{eq:3ObjFunc}  \\
\hline
Div. 1 FBS and FCS&0.7015 & 8.780 & 2.363  \\
Erd\"os-R\'enyi, $n=246$  &2.892& 9.681 & 2.358 \\
E-optimal design, $n=246$ & {\bf  6.630} & {\bf 10.71}&{\bf  2.403}  \\
\hline
Div. 1 FBS &1.725 &  9.634 & 2.372  \\
Erd\"os-R\'enyi, $n=119$ & 3.497& 9.911 & 2.361 \\
E-optimal design, $n=119$ & {\bf 7.142} & {\bf 10.92 }&{\bf 2.402}
\end{tabular}
\end{center}
\caption{A comparison of the three objective functions defined in \eqref{eq:3ObjFun} for the Division 1 and Division 1 FBS schedules, Erd\"os-R\'enyi random schedules, and schedules which nearly-maximize $J_{E}(w) = \lambda_{2}$. Schedules which nearly-maximize $J_{E}(w) = \lambda_{2}$ also have larger values of $J_{A}$ and $J_{D}$ than the comparison schedules. 
See \S \ref{sec:football}
 \label{tab:EADvals}}
\end{table}


\subsection{Synthetic data experiment on the 2011-2012 NCAA Division 1 FBS graph} \label{sec:sythExp}
To further illustrate and test our proposed active learning method, we again consider the graph generated in \S \ref{sec:football} from the 2011-12 NCAA Division I Football Bowl Subdivision (FBS)  schedule with $n=119$ nodes and $m=693$ edges, as shown in Figure \ref{fig:cluster1A}.  We take as ground truth rating, $\phi$, a normally distributed vector with mean zero and variance, $\sigma^2= 1$. 
The ground truth rating, $\phi$, is used to generate new data according to the normal model \eqref{eq:errorModel} with $\sigma^2 = 5$. With this data, we solve \eqref{eq:LSE} to obtain a least squares estimate, $\hat \phi_{w_0}$. We compute $\| \hat \phi_{w_0} - \phi \|_2 = 17.31$ and $K(\phi_{w_0}, \phi) = 0.35$. Here,  
the \emph{Kendall-$\tau$ rank distance} between two rankings $\phi_1$ and $\phi_2$ is defined as the fraction  of pairwise disagreements between the rankings, 
\begin{equation} \label{eq:KTau}
K(\phi_1,\phi_2) := \frac{ \# \{(i,j) \colon i>j, \ \phi_1(i)< \phi_1(j),  \text{ and } \phi_2(i) > \phi_2(j) \}  }{n(n-1)/2}.
\end{equation}

\begin{figure}[t!]
\begin{center}
\includegraphics[width=.34\textwidth]{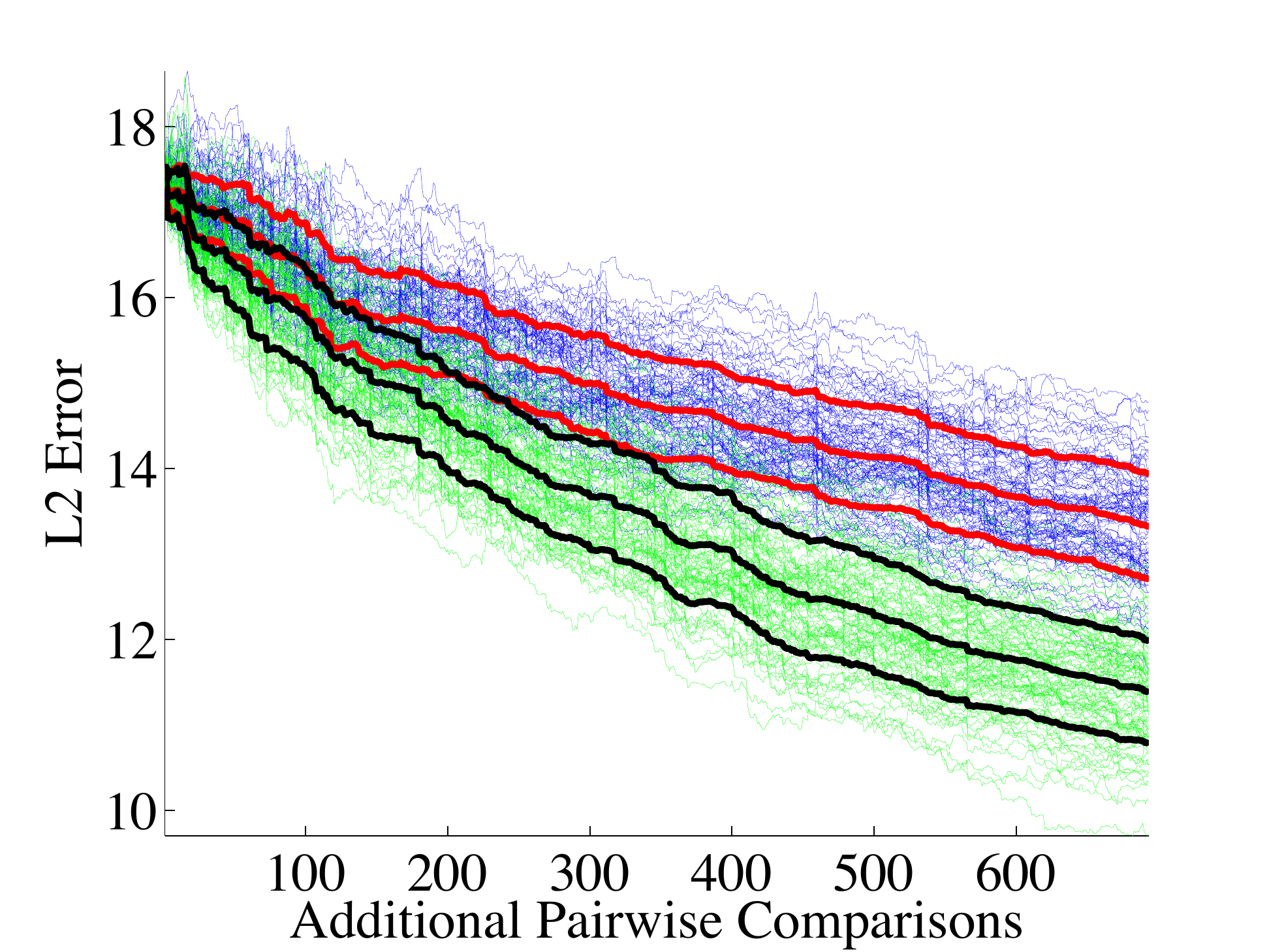} \hspace{-.5cm}
\includegraphics[width=.34\textwidth]{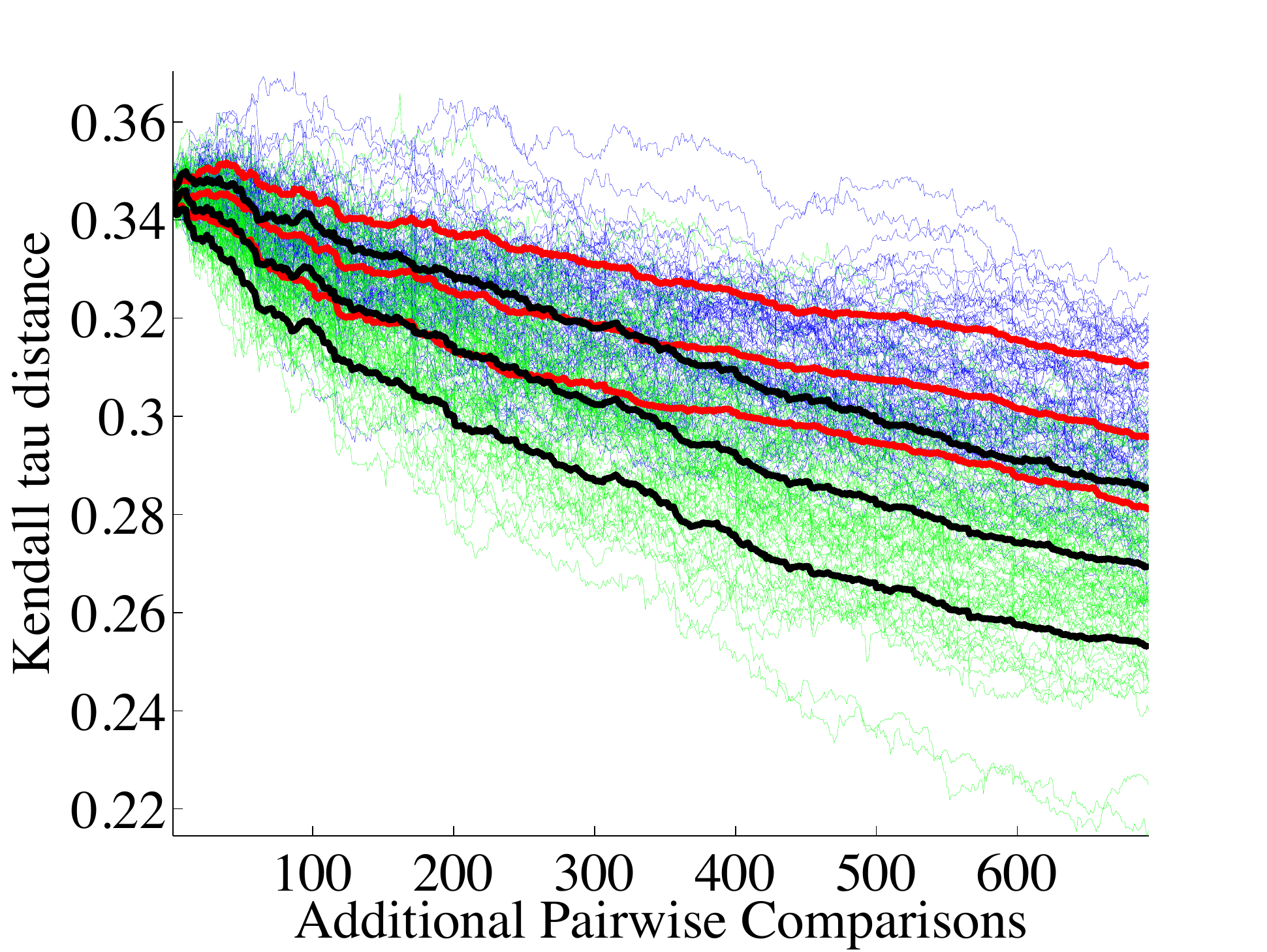} \hspace{-.5cm}
\includegraphics[width=.34\textwidth]{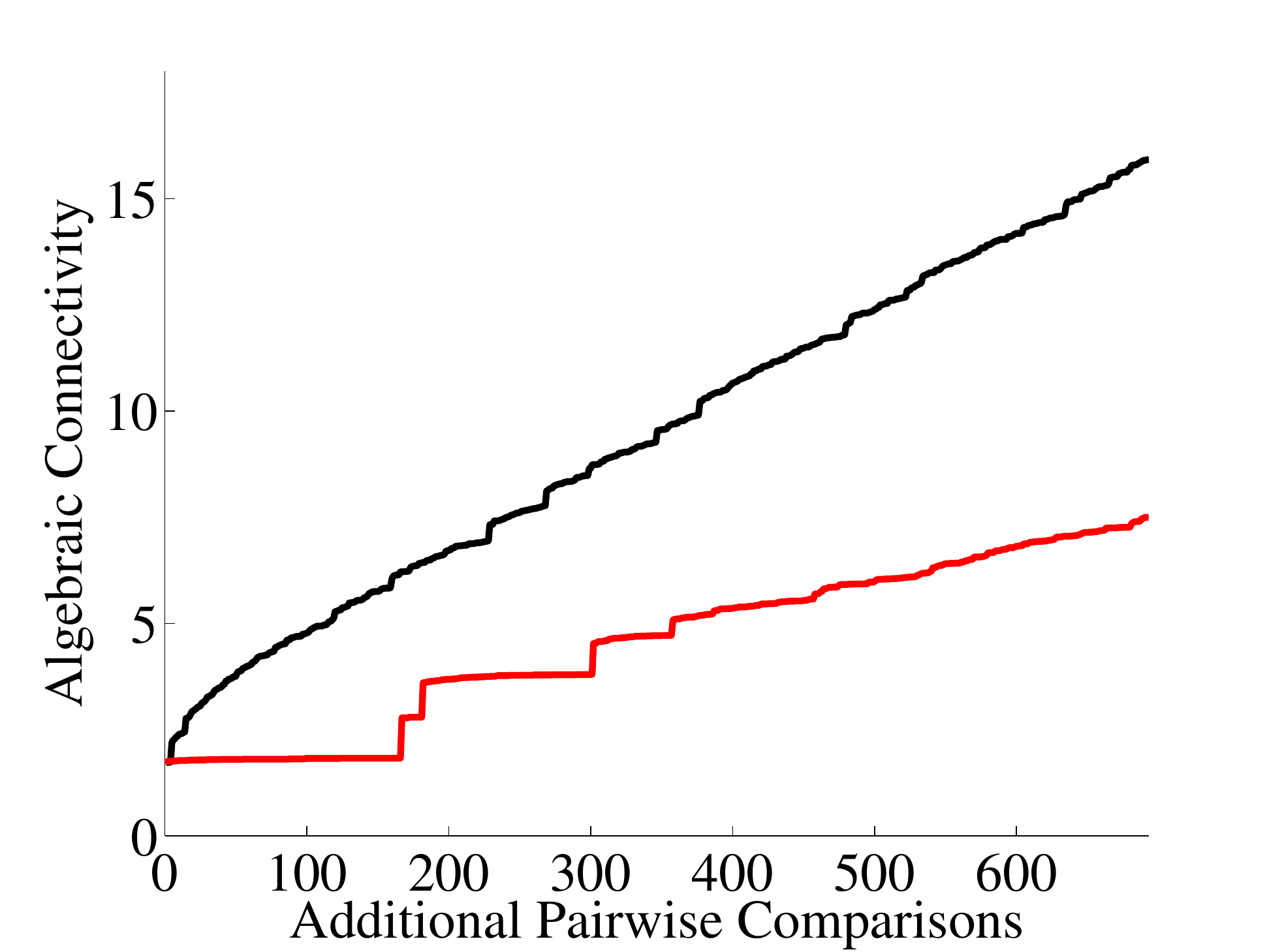}
\caption{A comparison of ranking errors and algebraic connectivity, a measure of the informativeness of  the ranking, for two data collection strategies: the proposed active sampling method (black and green) and random sampling (red and blue). 
{\bf (left)} The $L^2$-error, $\| \hat \phi_{\xi} - \phi \|_2$ between the estimated and ground truth rankings.  
{\bf (center)} The Kendall-$\tau$ rank distance, \eqref{eq:KTau}, between the estimated and ground truth rankings. 
{\bf (right)}  The algebraic connectivity of the graph representing the dataset. 
See \S \ref{sec:sythExp}.
\label{fig:synthData}}
\end{center}
\end{figure}

We then consider enhancing the dataset by adding $\xi$ more pairwise comparisons. Using the enhanced dataset, we compute an estimate of the ranking, $\hat \phi_\xi$, and, as $\hat \phi_\xi$ is an unbiased estimate of $\phi$, expect  $\| \hat \phi_{\xi} - \phi \|_2$ to diminish as $\xi \to \infty$. 
We choose  $\xi  = 693$, so that the number of pairwise comparisons (games played) is doubled. As in \S \ref{sec:movies}, we add pairwise comparisons either by the greedy algorithm (Algorithm \ref{alg:Greedy}) or by random selection. As with the data collected on the initial graph, the new data are collected according to the normal model \eqref{eq:errorModel} with $\sigma^2 = 5$. 
In  Fig. \ref{fig:synthData}(left), we plot the number of additional pairwise comparisons vs.  the $L^2$-error, $\| \hat \phi_{\xi} - \phi \|_2$, for an ensemble of ranking estimates determined using the two data collection strategies. The (thin) blue and green lines represent the error for 100 instances of data collection using the random and greedy methods respectively. The (thick) red and black lines represent the mean and mean plus/minus one standard deviation for each of the two data collection strategies. For $\xi = 693$, the mean $L^2$-error for the proposed data collection strategy is 11.38 while the mean error for the random data collection strategy is 13.32,  representing a reduction in error of $34\%$ and $23\%$ respectively. 
In  Fig. \ref{fig:synthData}(center), we plot the number of additional pairwise comparisons vs.  the Kendall-$\tau$ rank distance, $K(\phi,\phi_\xi)$, for these two data collection strategies.  For $\xi = 693$, the mean distance for the proposed data collection strategy is .27 while the mean error for the random data collection strategy is 0.30, representing a reduction in distance of $22\%$ and $14\%$ respectively.
In  Fig. \ref{fig:synthData}(right), we plot the algebraic connectivity, a measure of the informativeness of the ranking, vs. the number of additional pairwise comparisons. The black (red) line is the algebraic connectivity for the graph representing the dataset where edges are added using the greedy algorithm (random sampling). Supported by Propositions \ref{prop:FI} and \ref{prop:minLam2}, the dataset represented by a graph with larger algebraic connectivity is more informative and thus produces a ranking estimate with greater fidelity to the ground truth estimate.


\section{Discussion and future directions}\label{sec:disc}
We have applied methods from optimal experiment design to provide a new framework for  data collection for more informative statistical rankings. At the heart of this framework is a bi-level optimization problem \eqref{eq:bilevelOpt} where the inner problem is to determine the unbiased ranking for a given schedule and the outer problem is to identify data which maximizes the Fisher information of the ranking.  For the least-squares estimate,  the outer problem decouples from the inner problem and reduces to an eigenvalue optimization problem. For the E-optimality criterion for the Fisher information, this is the problem of finding an edge weight $w\in \mathbb Z_+^N$, such that the $w$-weighted graph Laplaican has large second eigenvalue \eqref{eq:minLam2}. This can be interpreted as finding a multigraph with large algebraic connectivity, a problem which has been well-studied in graph theory. 
In the case of NCAA Division 1 football, we demonstrated  in \S \ref{sec:football} and Table \ref{tab:EADvals} that the nearly-optimal data collection strategy in the sense of E-optimality is also a good startegy in the sense of D- and A-optimality; the choice of  scalar function $f\colon \mathbb S_{+}^{n} \to \mathbb R$ as defined in \eqref{eq:bilevelOpt}  does \emph{not} strongly effect the optimal data collection strategy (see Remark \ref{rem:OtherOptCond} for a further discussion of these  optimality criteria). Furthermore, in \S \ref{sec:sythExp}, we demonstrate using a synthetically constructed dataset on this graph that the ranking estimate obtained via active sampling has greater fidelity to ground truth than the ranking estimate obtained  via random sampling. 

There are several applications in, \eg,  social networking, game theory, and e-commerce, where improved data collection could potentially benefit ranking. In particular, for the Yahoo! Movie user ratings dataset (considered in \S \ref{sec:movies}), we have shown that the informativeness of ranking can be increased by a factor of $2.2$ if just $.01\%$ of additional optimally-targeted pairwise comparisons are added to the dataset. In contrast, if the same amount of random data is added, there is an unappreciable effect on the informativeness of the ranking.  For this application, the data collection problem could be more carefully modeled. Here, the pairwise comparison data is constructed from user rating data and thus any targeted pairwise comparison addition must be solicited from a user. Since the number of pairwise comparisons for which a particular reviewer adds when a new movie is reviewed is equal to the number of previous reviews that user has contributed, it may make sense to solicit additional reviews from users with many previous reviews.  That is,  the propagation of information from the user reviews to the pairwise comparison data in \eqref{eq:constructPairwise} should also be considered.

We have focused on optimal data collection for improved rankings, neglecting several important  factors including the cost of data collection and potential constraints on what data may be collected. There are two simple extensions to our method which may be employed to accommodate these additional factors.  
The cost of data collection could be incorporated by either adding a penalization term in \eqref{eq:minLam2} or by incorporating additional weights into the norm used to compute $\lambda_{2}$ in \eqref{eq:minLam2}. 
Data collection constraints may be handled by explicitly forbidding certain edge weights to be incremented in the greedy Algorithm \ref{alg:Greedy} for targeting data collection. 

The least-squares ranking estimate  \eqref{eq:R}   is referred to as HodgeRank by some authors \cite{Jiang2010,Xu2011}, 
 since the Hodge decomposition implies that the residual in \eqref{eq:R}, $r = B \phi - y$, can be further decomposed into two orthogonal components: (1) a divergence-free component which consists of 3-cycles and (2) a harmonic component which consists of longer cycles \cite{Jiang2010,Hirani:2011fk}. In fact, \citet{Jiang2010} argues that a dataset which has a large harmonic component is inherently inconsistent and does not have a reasonable ranking. The harmonic component lies in the kernel of the graph Helmholtzian with dimension given by the first Betti number of the associated simplical complex. Optimal reduction of the first Betti number  may provide an alternative approach to improving the informativeness of the least squares ranking.

Recently, \citet{Masuda2013} developed an algorithm for removing nodes from a graph to increase the algebraic connectivity. This algorithm could be used to prune the alternatives in a dataset to increase the informativeness of a ranking. 

Finally, we are interested in extending this work to  nonlinear ranking methods, including robust estimators  \citep{Osting:2012fk},  random walker methods  \citep{Callaghan:2007},  Perron-Frobenius eigenvalue methods \citep{Keener:1993zr,Langville:2012}, and Elo methods \citep{Elo1978,Glickman:1995,Langville:2012}. 

\acks{We thank Lawrence Carin, J\'er\^{o}me Darbon, Mark L.  Green, and Yuan Yao for useful discussions. B. Osting is supported by NSF DMS-1103959. C. Brune is supported by ONR grants N00014-10-10221 and N00014-12-10040. S. Osher is supported by ONR N00014-08-1-1119, N00014-10-10221, and NSF DMS-0914561.}

\vskip 0.2in
\bibliography{sd.bib}
\end{document}